%% file: main.tex
\documentclass{article}

\usepackage{xcolor}
\usepackage{natbib}
\usepackage[hang]{subfigure}
\usepackage[utf8]{inputenc} 
\usepackage[T1]{fontenc}    
\usepackage{hyperref}       
\usepackage{url}            
\usepackage{booktabs}       
\usepackage{amsfonts}       
\usepackage{nicefrac}       
\usepackage{microtype}      
\usepackage{amsmath,amsfonts,mathtools}
\usepackage{amsthm}
\usepackage{afterpage}      
\usepackage{cleveref }
\usepackage{algorithm}
\usepackage{algorithmic}
\usepackage{authblk}
\usepackage{subfiles}
\usepackage{adjustbox}

\theoremstyle{plain}

\newtheorem{lem}{Lemma}

\theoremstyle{definition}

\theoremstyle{remark}

\newcommand\numberthis{\addtocounter{equation}{1}\tag{\theequation}}

\def\*#1{\boldsymbol{#1}}

\title{Scalable Probabilistic Matrix Factorization with Graph-Based Priors}

\author[1]{Jonathan Strahl}
\author[2]{Jaakko Peltonen}
\author[1,3]{Hiroshi Mamitsuka}
\author[1]{Samuel Kaski}\affil[1]{Helsinki Institute for Information Technology HIIT,
\protect\\ Department of Computer Science, Aalto University}
\affil[2]{Faculty of Information Technology and Communication Sciences, Tampere University, Finland}
\affil[3]{Bioinformatics Center, Institute for Chemical Research, Kyoto University, Japan}

\begin{document}

\maketitle


\begin{abstract}
In matrix factorization, available graph side-information may not be well suited for the matrix completion problem, having edges that disagree with the latent-feature relations  learnt from the incomplete data matrix. We show that removing these \textit{contested} edges improves prediction accuracy and scalability. We identify the contested edges through a highly-efficient graphical lasso approximation. The identification and removal of contested edges adds no computational complexity to state-of-the-art graph-regularized matrix factorization, remaining linear with respect to the number of non-zeros. Computational load even decreases proportional to the number of edges removed. Formulating a probabilistic generative model and using expectation maximization to extend graph-regularised alternating least squares (GRALS) guarantees convergence. Rich simulated experiments illustrate the desired properties of the resulting algorithm. On real data experiments we demonstrate improved prediction accuracy with fewer graph edges (empirical evidence that graph side-information is often inaccurate). A 300 thousand dimensional graph with three million edges (Yahoo music side-information) can be analyzed in under ten minutes on a standard laptop computer demonstrating the efficiency of our graph update.
\end{abstract}

\section{Introduction}\label{sec:intro}

Matrix factorization (MF) is popular in a number of domains including recommender systems \cite{koren2009matrix,mehta2017review}, bioinformatics \cite{brunet2004metagenes,jacoby2018future,stein2018enter,zakeri2018gene,zheng2013collaborative}, image restoration \cite{xue2017depth} and many more \cite{davenport2016overview}. Much of the data is of a very large scale and sparse, and additional (side-)information is usually available. Therefore, many methods focus on scalability \cite{davenport2016overview,mnih2008probabilistic,sardianos2019optimizing} and the addition of side information (SI) \cite{chiang2015matrix,Chiang2018NoisySI,gonen2013kernelized,ma2011recommender,zakeri2018gene,zhou2012kernelized,zhao2015expert}, and more recently scalable methods with SI \cite{monti2017geometric,rao2015collaborative,yao2018convolutional}.

Empirical evidence shows that prediction accuracy is significantly improved by graph SI, where edges in the graph represent similarity between connected nodes \cite{cai2011graph,ma2011recommender,monti2017geometric,rao2015collaborative,yao2018convolutional,zhou2012kernelized,zhao2015expert}. MF (or low-rank matrix completion) has theoretical guarantees for exact completion without and with noise \cite{candes2010matrix,candes2009exact}. 
Introducting noisy SI is shown to reduce sample-complexity, and is reduced even further handling the noise \cite{chiang2015matrix}. Reduction in sample complexity through the introduction of graph SI has also been shown \cite{ahn2018binary,rao2015collaborative}, as a function of graph quality.  However, to the best of our knowledge there is no work on scalable methods to handle the noise in the graph SI.

Mnih and Salakhutdinov  \cite{mnih2008probabilistic} introduced probabilistic matrix factorisation (PMF), which is equivalent to $\ell_2$-regularised (alternating least squares) MF. Probabilistic interpretations for MF with graph SI are kernelized PMF (KPMF  \cite{zhou2012kernelized}) and kernelized Bayesian MF (KBMF \cite{gonen2013kernelized}): placing priors over the columns of the latent feature matrices. This type of prior models the pairwise relation between rows, where these rows correspond to rows or columns of the incomplete data matrix. KPMF and KBMF showed good results on moderate-sized data but failed to scale to large data. 

To address scalability, graph-regularised least squares (GRALS \cite{rao2015collaborative}) was  proposed, with conjugate gradient descent exploiting the sparsity in the data matrix and the graphs, resulting in linear computational complexity and fast convergence. Recently there has been progress on applying deep learning to matrix completion, with and without side information, with good accuracy and showing potential for scalability \cite{berg2017graph,hartford2018deep,monti2017geometric,yao2018convolutional}.

All of the non-Bayesian or scalable methods incorporating graph SI \cite{cai2011graph,ma2011recommender,monti2017geometric,rao2015collaborative,zhou2012kernelized} fix the edges in the graph, considering them as true.  However, these graphs are known to be uncertain \cite{adar2007managing,asthana2004predicting}, and furthermore, the similarities they represent (e.g. homophily \cite{mcpherson2001birds}) are rarely specific to the matrix factorization task leaving no guarantee that correlations correspond \cite{ma2011recommender,singla2008yes}; graphs are often formed for other purposes, and hence their usefulness for MF is uncertain. This leaves room for improving the quality of the graph, leading to a significant reduction in sample complexity \cite{ahn2018binary}. In this work we will introduce a solution based on contested edges, defined later in the paper.

\paragraph{Example of Graph Side-Information and Contested Edges}
To better understand how graph similarities are not task-specific (are non-specific) to MF, take a common example of a movie-recommendation problem with social network (SN) SI (\citet{ma2011recommender} and in our experiments on Douban data). Connected users in the SN do not connect based on their similar preference of movies, instead they connect on the basis of a broader social context.
Similarly, the demographic information in MovieLens\footnote{https://grouplens.org/datasets/movielens/}, used to form a user-similarity graph, is only very indirectly related to the movie preferences \cite{mcpherson2001birds}. Nevertheless, more general similarity has been shown to often work well in practice, but some parts of it may turn out to be detrimental as we illustrate below.
 
\Cref{fig:CombinedIllustration} (top) 
shows a small movie-recommendation data matrix with SN SI (bottom-left).  Without SI, if row/column observations in the data matrix are similar, latent features will be similar. This can be inaccurate, e.g. users 2 and 3 would be considered similar based on the observations, and thus predictions for user 2 would be similar to ratings of user 3, whereas actually user 2 is similar to user 1. Graph information can help by encouraging latent features of connected users, like user 1 and user 2 here, to be similar, even when there is no observed data in the matrix to indicate they should be. However, for other users such as 4 and 5 the graph may mismatch with the data, indicating similarity whereas 4 and 5 are actually negatively correlated (as seen in their ratings of movies 5 and 6), and using the graph would thus worsen their predictions. We propose using this discrepancy to \textit{contest} the graph edge between users 4 and 5; removing this edge as in \Cref{fig:CombinedIllustration} (bottom-right) would improve predictions for users 4 and 5 to be consistent with their observed negative correlation, while the beneficial edge between users 1 and 2 will still remain. In real cases, mismatch between the data matrix and the SI would be detected based on much more data than in this illustration.


\begin{figure}[t]
\begin{tabular}{ccc}
  \begin{minipage}[c]{0.6\textwidth}
  \vspace{0pt}
  \begin{tabular}{llllllll}
    \toprule
    \multicolumn{1}{c}{} & \multicolumn{3}{l}{Movie}                   \\
    \cmidrule(r){2-8}
     User     & $m_1$ & $m_2$ & $ m_3$ & $m_4$ & $m_5$ & $ m_6$ & $ m_7$ \\
    \midrule
     $u_1$ &  5 &   &  1   &&&&\\
     $u_2$     & \textcolor{lightgray}{5} &  4 & \textcolor{lightgray}{1} &&&&\\
     $u_3$  &  1 & 4 & 5  &&&& \\
     $u_4$ &&&&  5 & 4 & 2 &  \textcolor{lightgray}{1}   \\
    $u_5$ &&&& \textcolor{lightgray}{1} &  2 & 4 & 5 \\
    \bottomrule
  \end{tabular}
  \end{minipage}
  & 
  \begin{minipage}{0.35\textwidth}
  \flushleft
  {\vspace{0.01cm}\includegraphics[width=0.7\textwidth]{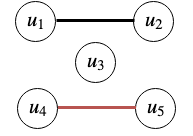}}
  \flushright
   {\vspace{0.01cm}\includegraphics[width=0.7\textwidth]{User1User2EdgeGraphSmoothBenefitPruned}}
   \end{minipage}
\end{tabular}
\caption{An illustrative movie recommendation problem.
\emph{Left:} data matrix where entries are user-ratings for movies: observations in black, unseen entries are blank and unseen entries to be predicted are in grey.
\emph{Middle:} Social Network SI; connected users assumed to have similar ratings. The edge shown in red is contested due to negative correlation of $u_4$ and $u_5$ in the data matrix. \emph{Right:} a graph update with removal of the contested edge to improve prediction accuracy.
}
\label{fig:CombinedIllustration}
\end{figure}


We do not propose to identify contested edges directly from the observed data but from correlations between the latent features. We introduce a probabilistic generative model that we call graph-based prior PMF (GPMF). Using the expectation-maximization (EM, \cite{bishop2006PRML}) algorithm we find a maximum a posteriori (MAP) estimate for the latent features and a maximum likelihood estimate (MLE) for the correlations of the latent features. 
We show in \Cref{sec:Mstep} how using GLASSO approximation we can remove contested edges by simply thresholding a constrained sample covariance matrix (SCM).

There exist a number of approaches to reduce the edges in a labelled graph, graph summarization, \citet{liu2018graph} for example. Most of these approaches do not use node attributes (labels) and to the best of our knowledge none use latent features for edge pruning. There are link prediction models that are probabilistic and use node attributes \cite{haghani2017systemic} but none of them can (yet) scale to large data \cite{li2014lrbm,nguyen2012latent,zhao2017leveraging}.

This paper introduces GPMF: the generative model in \Cref{sec:GenarativeModel}, the scalable constrained EM algorithm in \Cref{sec:TheEMAlgo}, experiments in \Cref{sec:experiments} and a conclusion in \Cref{sec:conclusion}.

\section{GPMF Generative Model and Relations to the Graph Side-Information} \label{sec:GenarativeModel}

We are provided with a partially observed data matrix $\*R$ with $N$ rows and $M$ columns. $\*R$ is approximated as the product of two low-rank matrices, $\*U$ and $\*V$. The number of latent features $D$ is fixed; $\*U$ and $\*V$ have $D$ columns, each row is a latent feature vector for each row / column of $\*R$ respectively. We use an index set $\*\Omega$ where $\*\Omega_{ij}$ is one if the element in row $i$ and column $j$ of $\*R$ is observed, and zero otherwise. The goal is to learn latent-feature matrices $\*U$ and $\*V$ that most accurately represent the full matrix $\*R$.

$\ell_2$-regularized MF has a scalable probabilistic interpretation: PMF. Each observed entry $\*R_{ij} : (i,j) \in \{ \*\Omega = 1 \}$ is assumed to have Gaussian noise $\sigma^2$; each row of $\*U$ and $\*V$ has a zero-mean spherical Gaussian prior.  Similar to KPMF \cite{zhou2012kernelized}, our model replaces the spherical Gaussian prior with a full-covariance Gaussian over the columns of the latent features (introducing row-wise dependencies):
\begin{align*}
p(\*R \mid \*U, \*V, \sigma^2) &= \prod_{i=1}^{N} \prod_{j=1}^{M} \mathcal{N} ( \*R_{ij} \mid \*U_{i:} \*V_{j:}^\top, \sigma^2 )^{\*\Omega_{ij}} \numberthis \label{eq:PmfLikelhood}\\
p(\*U \mid \*\Lambda_U) &= \prod_{d=1}^{D} \mathcal{N}(\*U_{:d} \mid \*0, \*\Lambda_U^{-1}) \numberthis \label{eq:GraphBasedPriorU}\\
p(\*V \mid \*\Lambda_V) &= \prod_{d=1}^{D} \mathcal{N}(\*V_{:d} \mid \*0, \*\Lambda_V^{-1}) \numberthis \label{eq:GraphBasedPriorV} \; .
\end{align*}
Graph SI constrains the structure of the precision matrices ($\*\Lambda_U$ or $\*\Lambda_V$) of \eqref{eq:GraphBasedPriorU} and \eqref{eq:GraphBasedPriorV},  discussed next. 
\subsection{Gaussian Markov Random Field (GMRF) relation to Precision matrix} An undirected graph $\mathcal{G}_Z = (\mathcal{V}_Z, \mathcal{E}_Z)$ with a set of nodes $\mathcal{V}_Z$, representing a set of random variables $\{Z_i\}_{i=1}^P$, and a set of edges $\mathcal{E}_Z \subseteq \{(i,j) \mid i,j \in \mathcal{V}_Z\}$, defines the conditional independence of the random variables, where the absence of an edge $(i,j) \notin \mathcal{E_Z}$ implies that the two random variables are conditionally independent $[\*\Lambda_Z]_{ij} = 0$ given the remaining random variables \cite{bishop2006PRML,hastie2009elements,lauritzen1996graphical,rue2005gaussian}: $Z_i \perp Z_j \mid \{Z_k : k \in (1,...,N) \setminus (i,j)\}$. In the remainder of the paper we refer to the adjacency matrix of $\mathcal{G}_Z$ : a symmetric matrix where $[\*A_Z]_{ij}$ is one if an edge exists between nodes $i$ and $j$ and zero otherwise. We can summarize the GMRF relation as $[\*A_Z]_{ij} = 0 \iff [\*\Lambda_Z]_{ij} = 0 \mid i \neq j$.
\subsection{Laplacian Matrix relation to Precision Matrix}  \label{sec:RegLaplaceMat} The Laplacian matrix of a graph is $\*L_Z = \*D - \*A_Z$, where $\*D_{i,i} = \sum_{j=1}^N [\*A_Z]_{ij}$ is a diagonal degree matrix, and is positive-semi-definite by definition.  The regularised Laplacian $\*L_Z^+ = \*L_Z + \gamma \*I \ , \ \gamma > 0$ is a positive-definite matrix; a valid precision matrix retaining the GMRF property \cite{dong2016learning,egilmez2016graph,egilmez2017graph,hastie2009elements,liu2014bayesian}: $\*[\*L_Z^+]_{ij} = 0 \iff [\*\Lambda_Z]_{ij} = 0 \mid i \neq j$.
\begin{lem} \label{lem:PosteriorEquivGRALS}
If the precision matrix in \eqref{eq:GraphBasedPriorU} and  \eqref{eq:GraphBasedPriorV} is the regularised Laplacian matrix $\*L^+_U,\*L^+_V$, then the MAP estimator of our model has the same objective function as GRALS \cite{rao2015collaborative}. Our GPMF model therefore gives a generalization of the GRALS objective function.
\end{lem}
\begin{proof}[Proof of Lemma \ref{lem:PosteriorEquivGRALS}.]
Our generative model is biconvex, and hence it suffices to prove for 
$\*U$ that the posterior is equivalent to the GRALS objective.
Holding $\*V$ fixed and finding the log posterior of $\*U$:
\begin{align*}
&\ln p(\*U | \*R, \sigma^2, \*V, \*\Lambda_U) 
\propto  \ln p(\*R \mid \*U, \*V, \sigma^2) p(\*U \mid \*\Lambda_U)\\
&\propto - \frac{1}{\sigma^2} \sum_{i=1}^{N} \sum_{j=1}^{M} \mathcal{P}_\Omega \left( \*R_{ij} - \*U_{i:} \*V_{j:} \right)^2 - \frac{1}{2} \sum_{d=1}^{D} \*U_{:d}^\top \*\Lambda_U \*U_{:d} \\
&= - \frac{1}{2} \lVert \mathcal{P}_\Omega (\*R - \*U \*V^\top) \rVert_{\text{F}}^2 - \frac{\sigma^2}{2} \text{tr}(\*U^\top \*L^+_U \*U )  ) \ \numberthis \label{eq:MAP_GRALS_Equiv},
\end{align*}
where $\*U_{i:}$ is row $i$ of matrix $\*U$ and $\*U_{:d}$ is column $d$ and noting that $\sum_{i,j} \*U^2_{ij} = \text{tr}(\*U^\top \*U) = \lVert \*U \rVert_{\text{F}}^2$. \Cref{eq:MAP_GRALS_Equiv} is the GRALS objective function \cite{rao2015collaborative}. Derivations in the supplementary material.
\end{proof}

\section{GRAEM: Scalable EM for GPMF} \label{sec:TheEMAlgo}

We naturally extend each least-squares sub-problem of GRALS \cite{rao2015collaborative} with graph-regularised alternating EM (GRAEM), having the same global convergence guarantees as GRALS \cite{xu2013block}.  We work through optimising $\*U$ with $\*V$ fixed, solving for $\*V$ has the same form.

\subsection{The EM Formulation}
We have an incomplete data matrix $\*R$, fixed matrix $\*V$, latent variable matrix $\*U,$ and graph SI. From the graph we derive $\*L_U^+$ (see \Cref{sec:RegLaplaceMat}), then set the precision matrix $\*\Lambda_U=\*L_U^+$, which we consider our model parameters. We want to maximize the expectation of the joint density of the data and the latent variables, with $\*U$ as our unknowns and $\*\Lambda_U$ as our input parameters:
\begin{align*}
\mathcal{Q}(\*\Lambda_U, \*\Lambda_U^{\text{old}}) &= \int_{\*U} p(\*U | \*R, \*\Lambda_U^{\text{old}}) \ln p(\*R,\*U \mid \*\Lambda_U) \; \textrm{d} \*U \\ 
&= \mathbb{E}_{p(\*U | \*R, \*\Lambda_U^{\text{old}})} \left[ \ln p(\*R,\*U \mid \*\Lambda_U ) \right]  \ . \numberthis
\label{eq:EMQFunc}
\end{align*}
\subsection{E-step: Expected Value of the Latent Variables} \label{sec:EStepFullU} The expected value of our latent variables has a Gaussian posterior distribution (see supplementary material), we can therefore use the MAP, which is equivalent to the GRALS objective function as shown in \Cref{lem:PosteriorEquivGRALS}: $ \mathbb{E}_{p(\*U | \*R, \*\Lambda_U^{\text{old}})}[ \ \*U \ ] = \*\mu_{U}^\text{post.} \approx \hat{\*\mu}^{MAP}_U  \label{eq:PoseMeanMAPApprox}$.

\subsection{M-step: Removing Contested Edges}\label{sec:Mstep}

We can remove edges in the graph that correspond to negative correlations between the latent features by simply removing negative covariances from an SCM; this relationship holds for large scale and sparse problems; details follow.
\subsubsection{The MLE of the parameters and GLASSO} To find the MLE we maximise the $\mathcal{Q}$ function in Equation \eqref{eq:EMQFunc} with respect to $\*\Lambda_U$. The maximum can be found in closed form by taking the derivative with respect to the parameter $\*\Lambda_U$ and setting to zero:
\begin{align*}
\underset{\*\Lambda_U}{\arg\max} \quad \mathcal{Q}(\*\Lambda_U, \*\Lambda_U^{\text{old}}) &= \left(\mathbb{E}_{p(\*U | \*R, \*\Lambda_U^{\text{old}})} \left[ \frac{1}{D} \sum_{d=1}^{D}  \*U_{:d} \*U_{:d}^\top \right] \right)^{-1} \\
&= \left(\mathbb{E} \left[ \*S_U^D \right] \right)^{-1} =  \*\Lambda_U^*  \numberthis  \ . \label{eq:MaxOfQForPrecU}
\end{align*}
\Cref{eq:MaxOfQForPrecU} is the inverse of an SCM, where each sample is one of the columns of $\*U$. Values for $\*U$ are unknown, so we use the MAP given the previous estimate of the parameters ($\*\Lambda_U^{\text{old}})$.  The solution (if any) is almost surely not sparse. Graphical lasso (GLASSO \cite{mazumder2012graphical}) finds a sparse solution for the MLE of the precision matrix, where samples are assumed to be normally distributed, in line with our model assumptions in \Cref{sec:GenarativeModel}. We therefore propose solving \eqref{eq:MaxOfQForPrecU} with GLASSO. 

\subsubsection{Constrained GLASSO and Highly Efficient Approximation} GLASSO finds the MLE of the precision matrix under an $\ell_1$ penalty, given an SCM $\*S$. \citet{grechkin2015pathway} showed that the problem space can be reduced with prior knowledge on which pairwise relationships do not exist, forcing them to be zero in the solution:
\begin{align*}
    \underset{\*\Lambda_U \succeq 0 }{\min}& \quad tr(\*S \*\Lambda_U  ) - \log | \*\Lambda_U | + \tau \left\lVert \*\Lambda_U \right\rVert_1 , \\  
    \text{subject to } & \quad \left[\*\Lambda_U\right]_{ij} = 0, \left[\*A^0_U\right]_{ij} = 0 \ . \numberthis \label{eq:ConstrainedGLASSO}
\end{align*}
\citet{zhang2018largescaleprec} uses a relation between the sparsity structure of the $\tau$-thresholded SCM and the GLASSO solution;  for large-scale problems, when the solution is very sparse, the connected components are equivalent \cite{mazumder2012graphical}, given further assumptions the complete sparsity structure is equivalent \cite{fattahi2019graphical,sojoudi2016equivalence,sojoudi2016graphical}. 
However, this solution will locate correlations, positive and negative, with a strong magnitude, greater than $\tau$. Next we detail how to identify edges that correspond to only negative correlations.

\subsubsection{Removing a Contested Edge}
The sparsity structure of the SCM and the (GLASSO) solution are equivalent under mild assumptions 
that are found to be true for sufficiently large $\tau$, that result in $\approx 10 N$ non-zeros in the solution \cite{fattahi2017graphical,fattahi2019graphical}. One of these assumptions is sign-consistency where each non-zero element of the solution has the opposite sign in the SCM. Assuming sign-consistency we can identify all graph edges that correspond to negative correlations in the latent features, with $\mathbb{E} [ \*S_U^D ]$ from \Cref{eq:MaxOfQForPrecU} as our SCM:

\begin{equation*}
    [\*A_U^{\text{new}}]_{ij} = \begin{cases} 
    1, & \left[\*A_U^0\right]_{ij} = 1 \ , \ \mathbb{E} \left[ \*S_U^D \right]_{ij} \geq \tau \\
    0, & \left[\*A_U^0\right]_{ij} = 1 \ , \ \mathbb{E} \left[ \*S_U^D \right]_{ij} < \tau \ , \hfill \textbf{ CE}\\
    0, & \text{otherwise} , \hfill \textbf{con-E} , 
    \end{cases} \numberthis \label{eq:posCovThesh}
\end{equation*}

where $\*A_U^{\text{new}}$ is the updated adjacency matrix, the threshold parameter $\tau$ is set to zero (or can be increased for a sparser solution) and $\*A_U^0$ is the adjacency matrix of the graph SI; CE is a contested edge and con-E is a constrained edge. To solve \Cref{eq:posCovThesh} we need to compute $\mathbb{E} [ \*S_U^D ]$, we can decompose the problem:
\begin{align*}
 \mathbb{E}[\*S^{D}_U] &= \frac{1}{D} \sum_{d=1}^D \mathbb{E}\left[\*U_{:d} \*U_{:d}^\top \right] \\
 \mathbb{E}\left[\*U_{:d} \*U_{:d}^\top \right] &=   \text{Cov}[\*U_{:d}] + \mathbb{E}[\*U_{:d}] \mathbb{E}[\*U_{:d}^\top] \\
&= \*\Sigma_{U_{:d}}^\text{post.} + \left[\*\mu_{U_{:d}}^\text{post.}\right]\left[\*\mu_{U_{:d}}^\text{post.}\right]^\top \ .
\end{align*}

The remaining task is to efficiently approximate the posterior covariance $\*\Sigma_{U_{:d}}^\text{post.}$ for each column, $d$, of $\*U$, which we discuss next.

\subsubsection{Posterior Covariance Approximation} \label{sec:SparsePrecEstMStep:ApproxPostCov}

The posterior of our GPMF model, in \Cref{sec:GenarativeModel}, is a joint Gaussian distribution, where the likelihood in \Cref{eq:PmfLikelhood} introduces relations between the columns of the latent features and the prior in  \Cref{eq:GraphBasedPriorU} introduces relations between the rows. This results in a posterior covariance matrix with an inverse Kronecker sum structure \cite{kalaitzis2013bigraphical,schacke2004kronecker}: $\*\Sigma_{U}^\text{post.} = (I_D \otimes \*\Lambda_U + \alpha \; \*C)^{-1}$ where $\otimes$ is the Kronecker product operator and 

\begin{align*}
\*C &= \left[ \*c(d,d') \right]_{d,d' = 1}^D \ , \\ \*c(d,d') &= \text{diag}\left( \left\{ \sum_{j=1}^{M}\*\Omega_{ij} \*V_{jd} \*V_{jd'} \right\}_{i=1}^N \right) \ .
\end{align*}


\paragraph{Column-wise independence assumption.} We simplify the Kronecker sum with a column-wise independence assumption, setting all off-diagonals of $\*C$ to zero:
\begin{align*}
\*\Lambda^\text{post.}_U  &\approx I_D \otimes \*\Lambda_U + \alpha \; \text{diag} \left( \*C \right)  \\ &= \text{blkdiag}\left( \left\{ \hat{\*\Lambda}_{U_{:d}}^{\text{post.}} \right\}_{d=1}^D \right)  \numberthis \label{eq:PostCovApprox} \;,
\\ 
\hat{\*\Lambda}_{U_{:d}}^{\text{post.}} &= \*\Lambda_U +\alpha \;  \text{diag} \left( \*C_d \right)
, \\
 \text{diag} \left( \*C_d \right) &=  \text{diag}\left( \left\{ \sum_{j=1}^{M}\*\Omega_{i,j} \*V_{j,d}^2 \right\}_{i=1}^N \right)  \ ,
\end{align*} 
where $\alpha = [\sigma^2]^{-1}$ is the inverse of the observation noise in \eqref{eq:PmfLikelhood}, diag takes a vector to create a diagonal matrix and blkdiag takes a sequence of matrices to construct a block-diagonal matrix.

\paragraph{Sparse Cholesky factorisation:} 

Each $\hat{\*\Lambda}_{U_{:d}}^{\text{post.}}$ is still too large to invert. Assuming the high-dimensional matrix is sparse, as in \citet{zhang2018largescaleprec}, its Cholesky factorisation is computable in $\mathcal{O}(N)$ time \cite{davis2004column}. We compute $K$ samples as an unbiased estimate for the approximate posterior covariance:
\begin{align*}
\hat{\*\Sigma}_{U_{:d}}^{\text{post.}} &= \left[ \hat{\*\Lambda}_{U_{:d}}^{\text{post.}}\right]^{-1} \approx \frac{1}{K} \sum_{k=1}^{K} \*x_k \*x_k^\top 
\\
\*x_k &\sim \mathcal{N}\left(\*0, \left[\hat{\*\Lambda}_{U_{:d}}^{\text{post.}}\right]^{-1} \right) \ .
\end{align*}

\begin{figure*}[t!]
\centering
\subfigure[Decreasing contested edges]{
  \includegraphics[width=0.45\columnwidth]{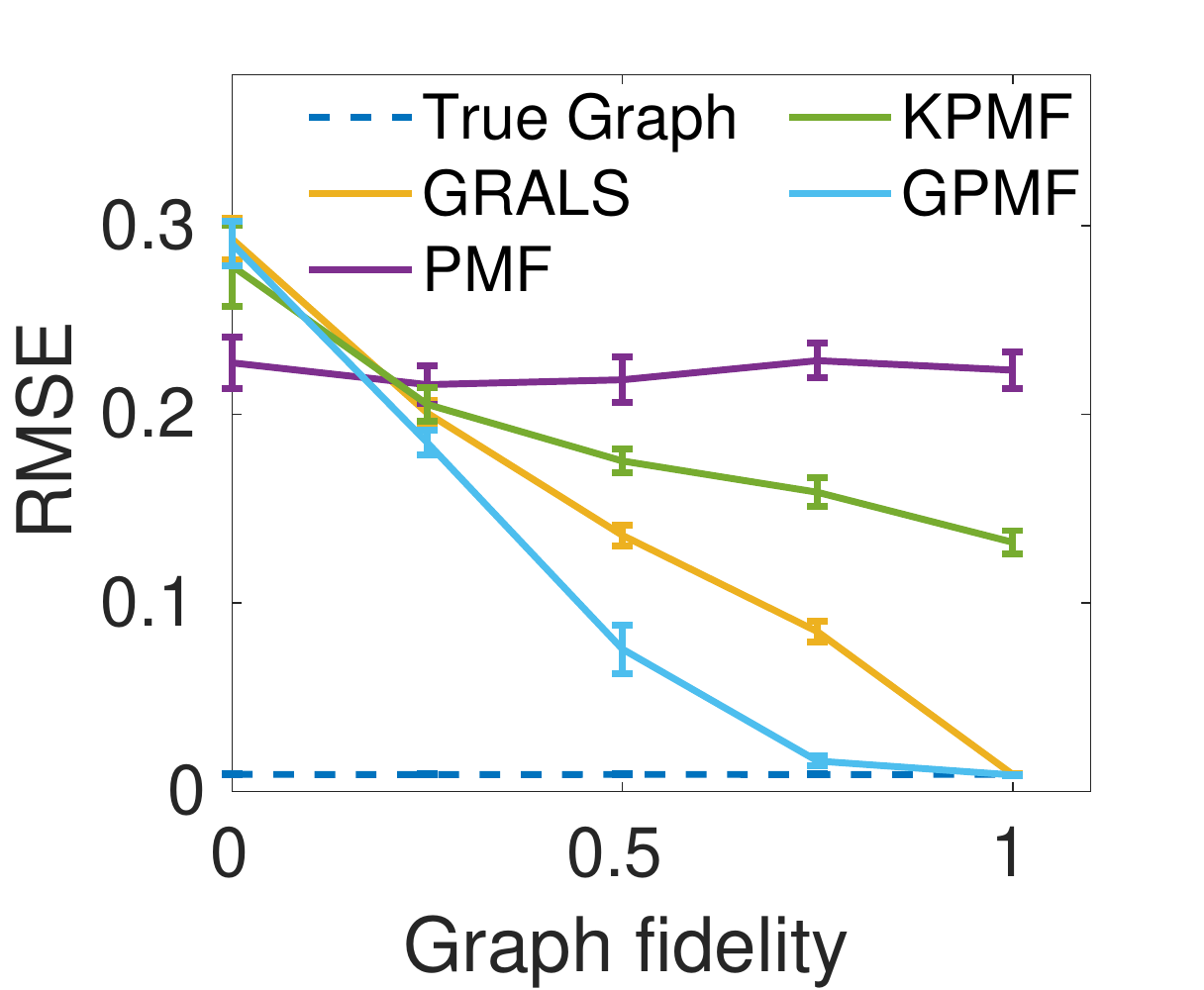}
}
\subfigure[Increasing observation noise]{
    \includegraphics[width=0.45\columnwidth]{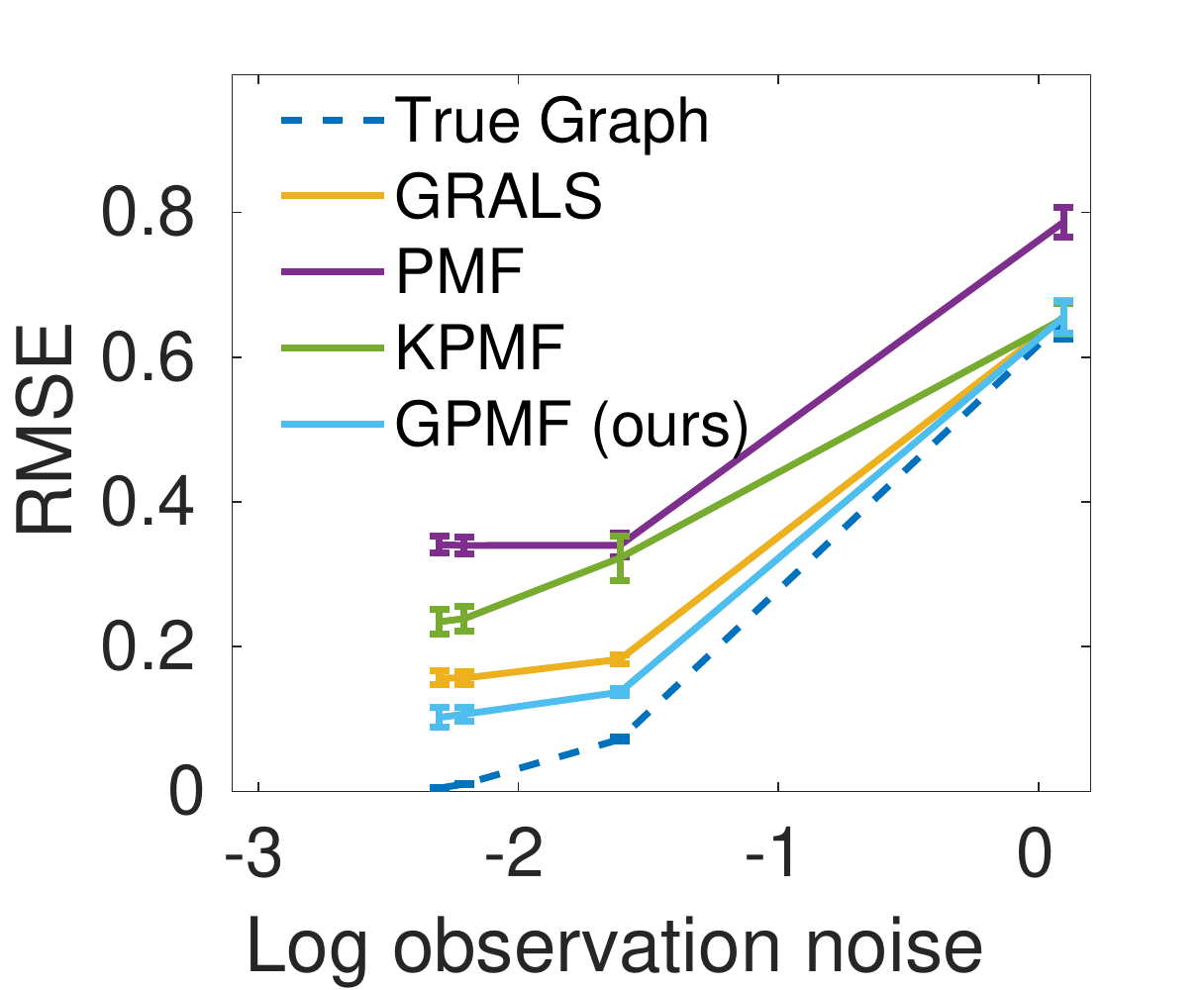}
}

\subfigure[Increasing observed entries]{
    \includegraphics[width=0.45\columnwidth]{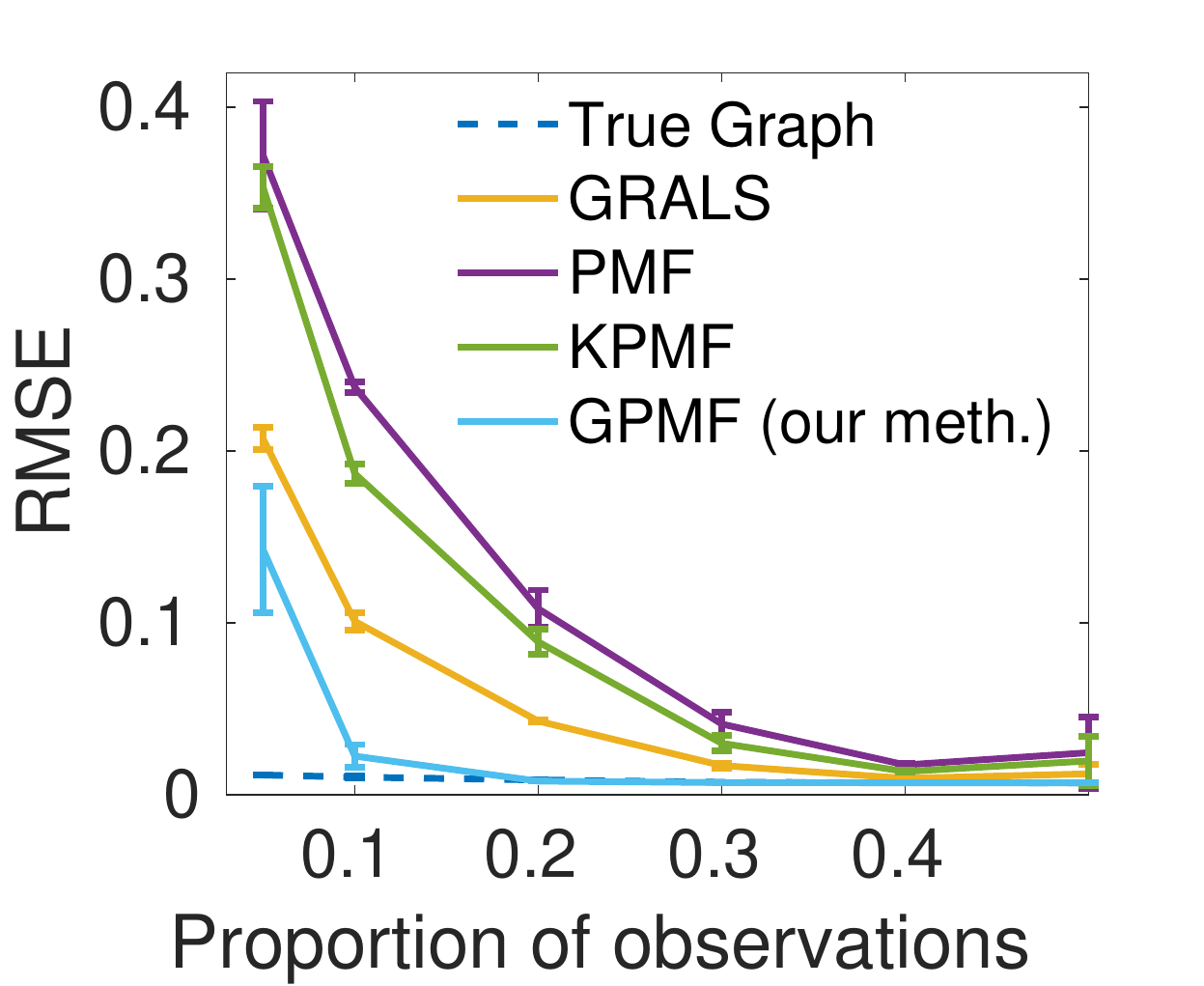}
}
\subfigure[Increasing dimensionality]{
    \includegraphics[width=0.45\columnwidth]{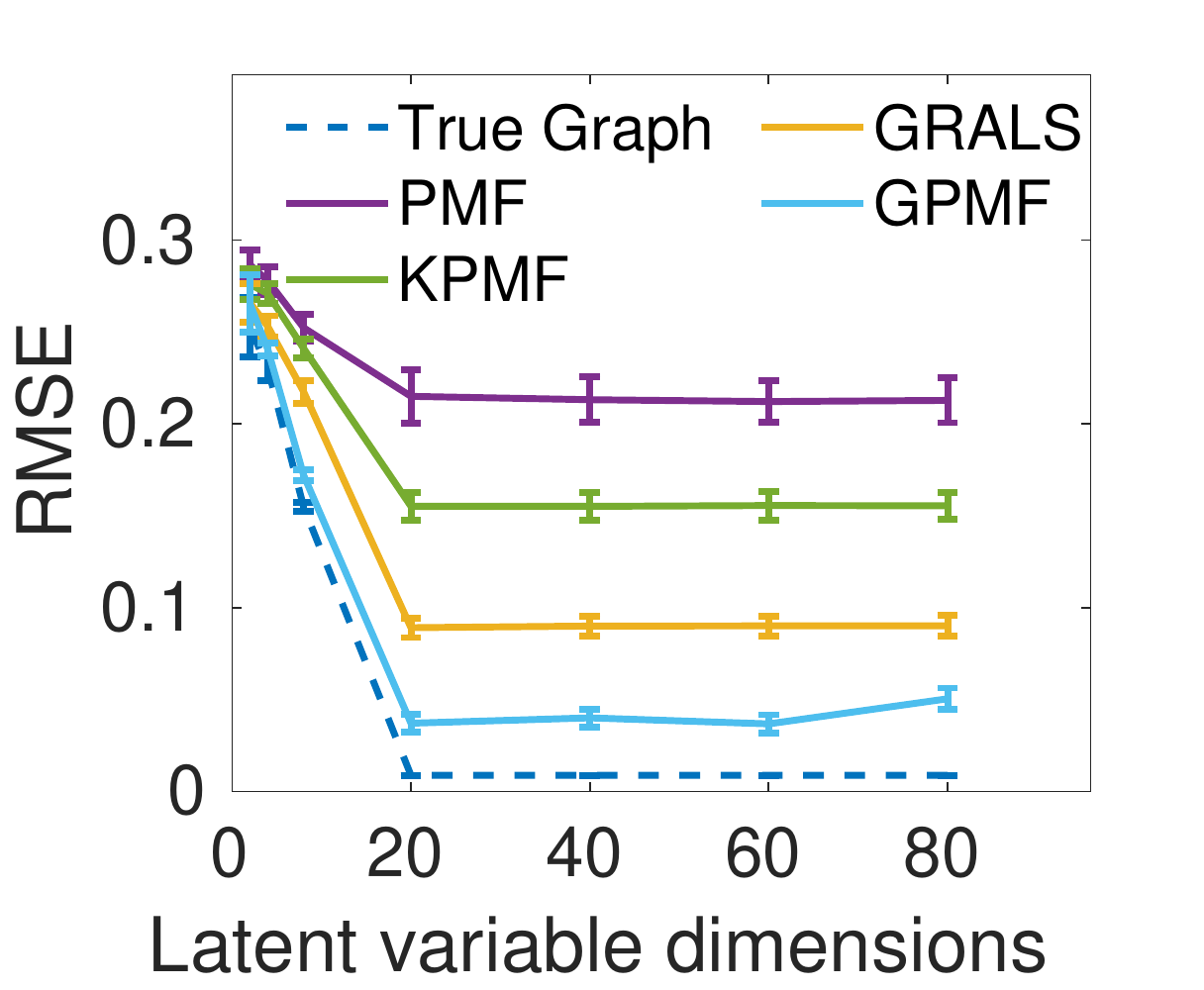}
}
\caption{Synthetic data experiments} \label{fig:SynthDataExps}
\end{figure*}

\subsection{The Algorithm}
The EM algorithm iterates between E-step and M-step until convergence.  We initialize the latent feature matrices ($\*U,\*V$) by finding the MAP with no graph SI using PMF, to learn latent features that reflect the observed entries of the data matrix. In practise any method to learn the latent features with no SI can be used. The M step uses the relations between the latent features to identify negative correlations and remove them from the graph SI. The E-step then finds the MAP of the latent features given the updated graph. In theory the E and M step could be continued until some convergence criterion was met, but this would be less efficient and we get good results with just one step. So the three steps of our algorithm are lines 1,3 and 4:
\begin{algorithm}
\caption{Graph-regularised alternating EM (GRAEM)}
\hspace*{\algorithmicindent} \textbf{Input: $\*A^{0}_U, \*A^{0}_V$} \\
 \hspace*{\algorithmicindent} \textbf{Output}: $\hat{\*U},\hat{\*V},\*A_U^+,\*A_V^+$
 \begin{algorithmic}[1]
\STATE $\*U^0,\*V^0 \gets \text{Initialise with PMF (GRALS with no graphs)}$
\WHILE{not converged}
\STATE $\*A^{t}_U, \*A^{t}_V$  $\gets$ \text{Run M-step \Cref{eq:posCovThesh} with } $\*U^{t-1},\*V^{t-1}$ and $\*A^{0}_U, \*A^{0}_V$ as structural constraints
\STATE $\*U^t, \*V^t$ $\gets$ \text{Run E-step with regularized Laplacians} \text{given } $\*A^{t}_U, \*A^{t}_V$
\ENDWHILE
\end{algorithmic} \label{algo:GPMF}
\end{algorithm}

\subsection{Scalability: Computational Complexity}

The algorithm has three steps: lines 1,3,4 in \Cref{algo:GPMF}. Line 1 is linear in the number of non-zeros $nz()$ in the data matrix $\mathcal{O}(nz(\*\Omega))$ per conjugate gradient (CG) iteration. Line 3 comprises sparse Cholesky factorisation, linear in time with respect to the dimension size $\mathcal{O}(N + M)$, constrained SCM computation and thresholding, $\mathcal{O}(nz(\*A_U) + nz(\*A_V))$ both converge in one time step. Line 4 uses GRALS with the sparsified graphs: $\mathcal{O}(nz(\*\Omega) + nz(\*A^+_U) + nz(\*A^+_V))$ per CG iteration. Line 4 is initialised with $\*U,\*V$ values from the PMF run, largely reducing the number of iterations required. Our algorithm remains linear with respect to the number of non-zeros. The additional M-step is a trivial additional cost, and if $\*A^+_U,\*A_V^+$ are much sparser, reducing iteration costs in Line 3, the overall computational load can be less than GRALS using the original graphs.


\section{Experiments} \label{sec:experiments}
We compare our algorithm to a baseline with no graph SI (PMF, \cite{mnih2008probabilistic}), the current most scalable method, GRALS \cite{rao2015collaborative}, and for accuracy less scalable methods KPMF \cite{zhou2012kernelized} and sRMGCNN \cite{monti2017geometric}. For sRMGCNN we used their published code, ran it on a (NVIDIA Tesla P100) GPU and used cross validation to find a good T value; note that this model took several orders of magnitude more time than the other methods: on Flixster data GPMF and GRALS converged in 20 seconds, PMF in 0.2 seconds, sRMGCNN took 30 minutes. We also ran KBMF \cite{gonen2013kernelized} but with an extremely long computational time on even the smallest dataset, and a large number of parameters, we failed to achieve reasonable results.

\subsection{Experiments on Synthetic Data} \label{sec:synthetic_experiments}
To analyze the behaviour of our algorithm we generate a data matrix with a known underlying graph.
Therefore we can replace true edges in the graph with \textit{corrupted edges} (CEs) that contest the true underlying structure, controlling the accuracy of the graph SI. We use a block-diagonal regularised-Laplacian precision matrix. 
We generate a $400\times 400$ data matrix by Equations \eqref{eq:PmfLikelhood}-\eqref{eq:GraphBasedPriorV}, with proportion of corrupted edges 0.3, observation noise 0.01, 7\% observed values, and 40 latent dimensions;
we vary these settings in the experiments below.
See supplementary material for further 
details. 

\textbf{Graph Fidelity.} 
In \Cref{fig:SynthDataExps} (a) we vary the number of CEs. A graph with no CEs has fidelity one ($F=1$), with all CEs $F=0$. GPMF consistently improves prediction accuracy over methods with graph SI for $F > 0$, and performance is equal for $F=0$. PMF with no graph performs better below $F=0.3$, showing that a graph of low quality can make prediction accuracy worse.

\begin{figure*}[t!]
\begin{adjustbox}{max width=1.5\textwidth,center}
\centering
\subfigure[Epinions (45\% of edges)]{\includegraphics[width=0.49\textwidth]{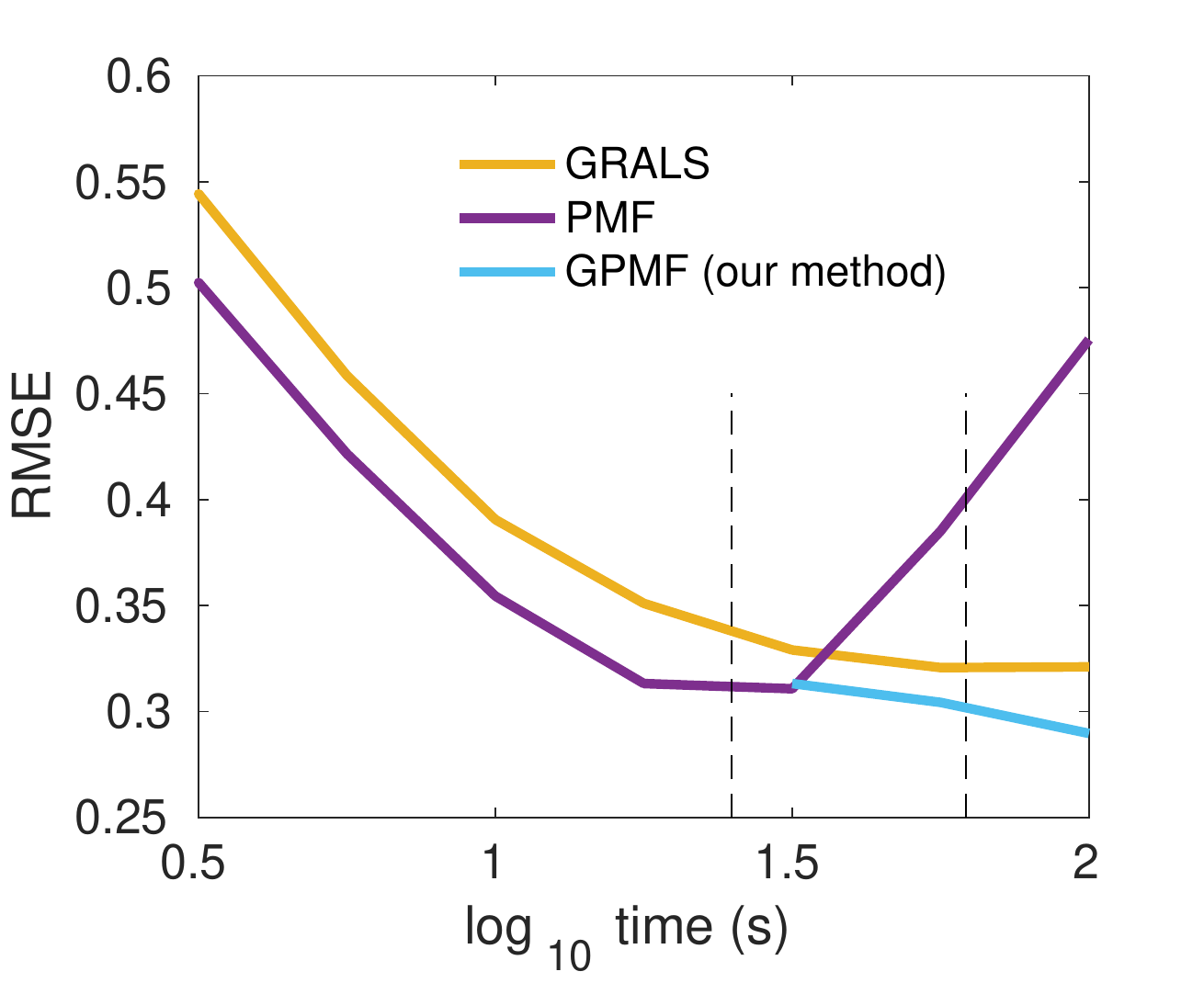}}
\subfigure[Yahoo Music (80\% edges)]{\includegraphics[width=0.49\textwidth]{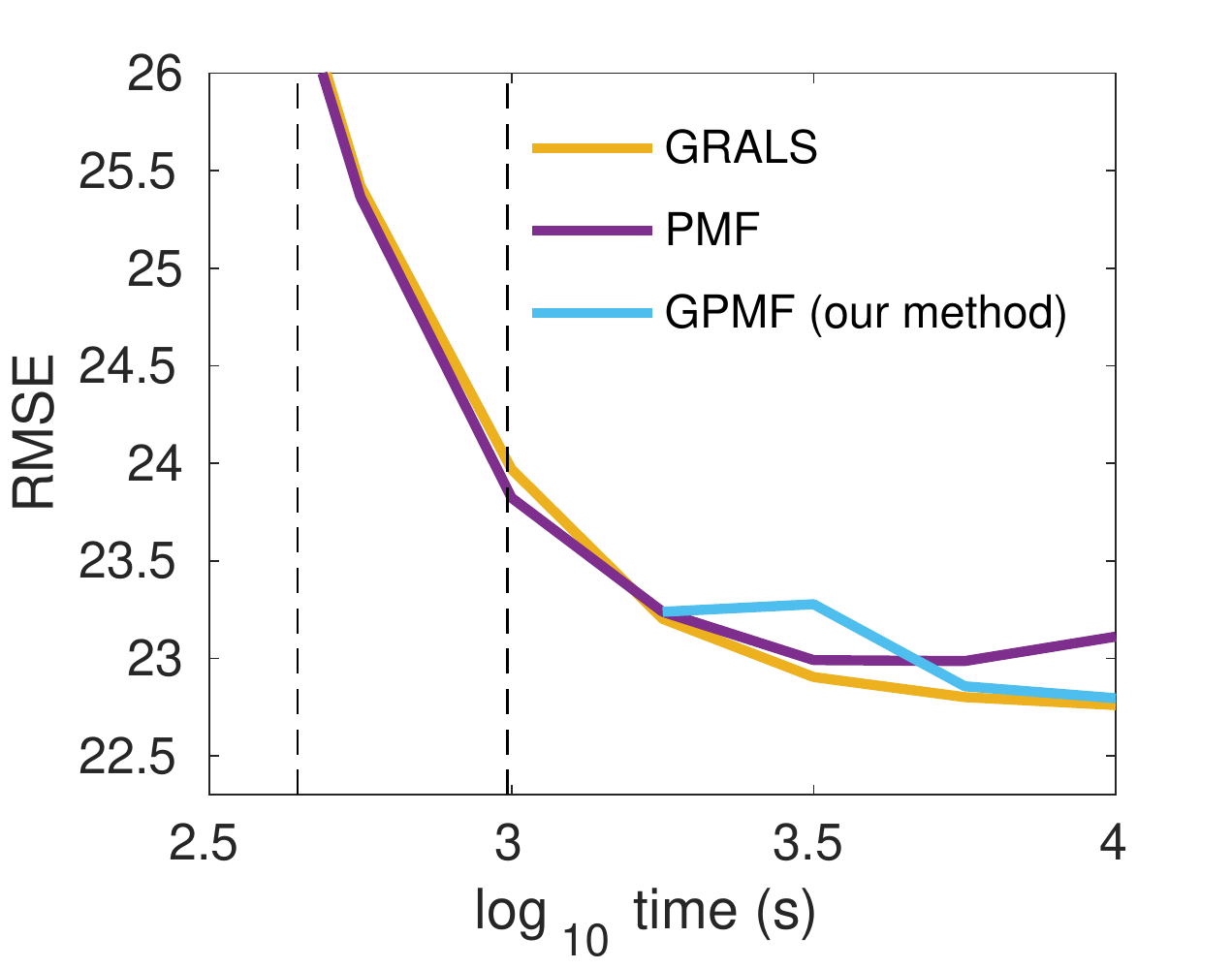}}
\subfigure[MovieLens20M (65\% edges)]{\includegraphics[width=0.49\textwidth]{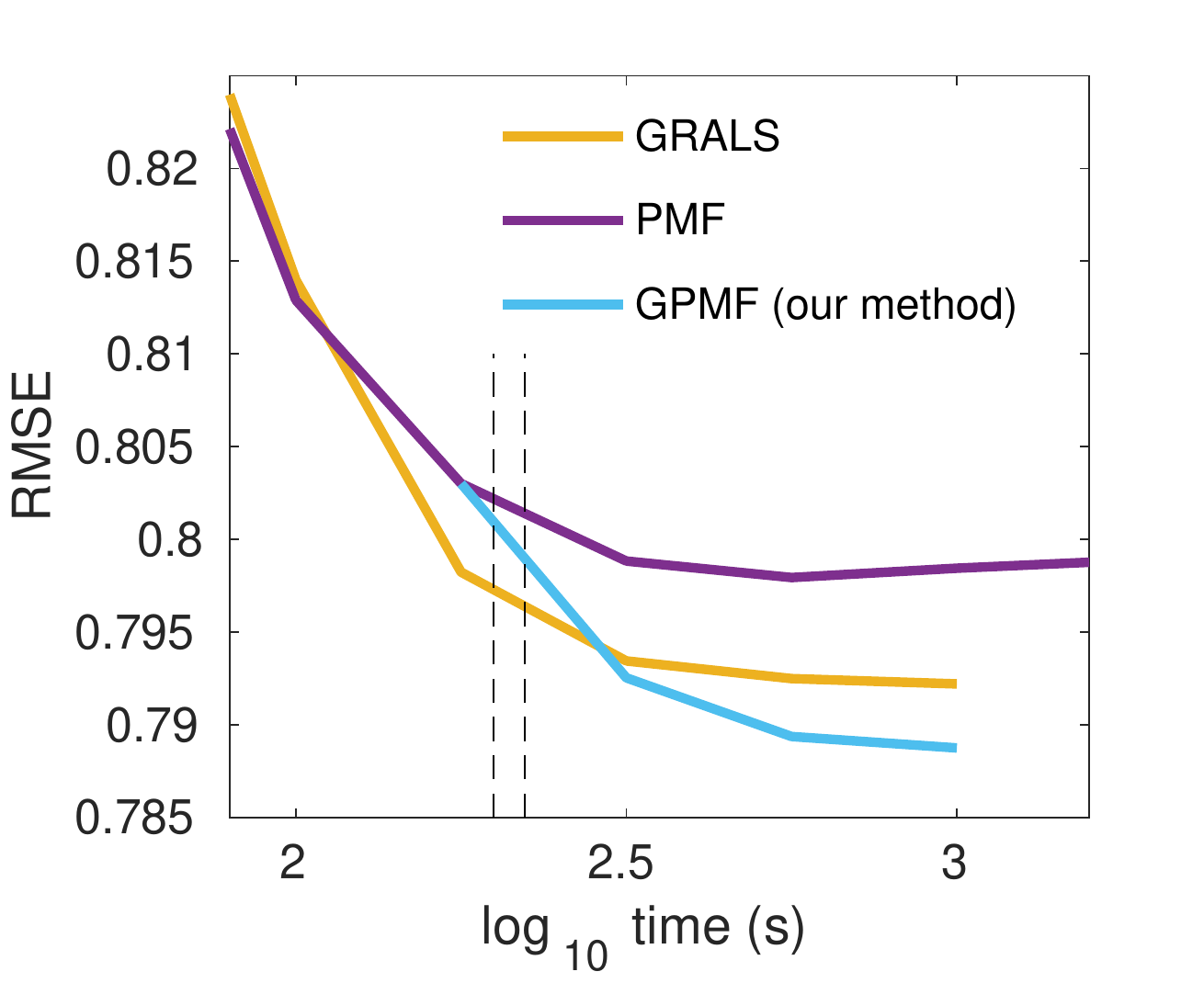}}
\end{adjustbox}
\caption{Convergence time; vertical lines show start and end of M-step. c) 40NN graph.}
\label{fig:LargeDataExperiments}

\end{figure*}

\textbf{Observation Noise.} 
\Cref{fig:SynthDataExps} (b)  shows the benefit of GPMF diminishes as noise increases; learning negative correlations requires learning from the observations. However, at worst GPMF is only as bad as using the original corrupted graph.

\textbf{Proportion of Observations.} 
In \Cref{fig:SynthDataExps} (c)  with just 10\% of observed entries our algorithm can almost attain the same prediction accuracy as using the true graph. GRALS requires 30\% to achieve a similar accuracy. At 40\% of observed entries the graph is no longer beneficial. Note that most large scale matrix completion problems have fewer than 10\% observed entries.

\textbf{Model Capacity.} \Cref{fig:SynthDataExps} (d)  shows that with too few latent features all models are negatively effected, but overall GPMF attains the best prediction accuracy.


\textbf{GLASSO accuracy} We analyse the accuracy of removing CEs over several simulations. With 7\% of observed entries, 31.7\% of CEs are correctly removed and 19\% of true edges (TEs) are wrongly removed; increasing observed entries to 40\%, 44.3\% of CEs are removed and 0.3\% of TEs. Fixing observed entries at 20\%, with noise $\sigma^2 = 0.01$, 39\% of CEs and 2.7\% of TEs are removed, and with $\sigma^2 = 1$, 34.3\% CEs and 42.7\% TEs are removed. We see clearly that observation noise strongly effects the ability to identify contested edges, as shown in \Cref{fig:SynthDataExps} (b).  Accuracy improves with more observed entries, but even with low levels of noise and a reasonable amount of observations successful removal of CEs is only moderate. Regardless of this moderate accuracy, experiments show this is enough to attain significant improvements in prediciton accuracy.


\subsection{Experiments on Real Data}
\label{sec:real_data_experiments}

In \Cref{tab:RMSEsComputationalComplexity} GPMF using GRAEM (our method) gives improved accuracy over GRALS on all small datasets: 3000 (3k) by 3k subsets of Flixster and Douban  \cite{monti2017geometric} (full datasets not attainable) and MovieLens100k \cite{harper2015movielens}); the bottom rows of the table show the size and number of observations for each data matrix and the number of edges in each side-information graph.   In \Cref{fig:LargeDataExperiments} our method is shown to add no computational cost on large data: MovieLens 20 million
\cite{harper2015movielens}, Epinions \cite{tang2012mtrust} and Yahoo Music \cite{rao2015collaborative,dror2011yahoo}), note that proportion of edges used by GPMF is reported in figure title. \Cref{fig:LargeDataExperiments} (a) is an example of poor quality graph side-information, we see this as PMF outperforms GRALS with the side-information; our method (GPMF using GRAEM) estimates over half the edges as contested, removing them seem to improve the quality. We believe that there were no gains in \Cref{fig:LargeDataExperiments} (b) as the graph is extremely sparse and removing some edges has little effect. We test this hypothesis with MovieLens 20M in \Cref{tab:RMSEsComputationalComplexity} by increasing the number of nearest neighbours from 10 to 40, we see that GRALS with the original graph decreases in performance while our algorithm continues to improve, we plot the best results in \Cref{fig:LargeDataExperiments} (c); the computational time to estimate contested edges (between the two vertical broken lines) is a fraction of the running time of the algorithm.

We also tested general usefulness of the updated graph: We get a small improvement for Douban with KPMF using with 77 \% of edges, we also get the same accuracy for Flixster with almost helf the edges.


\begin{table*}[h]

\caption{Result summary on real datasets (RMSE), $\*A^+$ is the graph updated with GRAEM (our method). Bold = best result.}
\label{tab:RMSEsComputationalComplexity}
\begin{footnotesize}
\begin{sc}
\hskip-3.2cm\begin{tabular}{lrrrrrr}
\toprule
&  Flixster &Douban&MovieLens& Epinions & Yahoo & MovieLens 20M \\
Algo.  &  (3k) &  (3k) & 100k &  &  Music &  (10-/20-/40-NN)\\
\midrule
PMF & 0.9809 & 0.7492  & 0.9728 & 0.31& 22.991 & 0.7980 / 0.7980 / 0.7980\\
GRALS & 0.9152 & 0.7504  & 0.9178 & 0.32& \textbf{22.760} & 0.7898 / 0.7925 / 0.7922\\
GPMF / GRAEM (\scriptsize{ours}) & \textbf{0.8857} & 0.7497 & \textbf{0.9174} & \textbf{0.28} & 22.795 & 0.7894 / 0.7895 / \textbf{0.7887}\\
KPMF & 0.9212 & 0.7324 & 0.9336 & - & - & - \\
KPMF ($\*A^+$)  & 0.9212 & \textbf{0.7323} & 0.9374 & - & - & - \\
sRMGCNN & 0.9108 & 0.7915  & 0.9263  & - & - & - \\
\midrule
Data dims. & \scriptsize{3k x 3k} & \scriptsize{3k x 3k}  & \scriptsize{1k x 1.5k}  & \scriptsize{22k x 296k} & \scriptsize{250k x 300k} & \scriptsize{138k x 27k} \\
num. of obs. &  2.6k & 137k &  100k & 824k & 6M & 20M \\
\midrule
 edges ($\*A_U$/$\*A_V$) & 59k / 51k & 2.7k / 0 & 12.6k / 29k  & 574k / 0 & 0 / 3M
 &  0 / 493k - 0 / 963k - 0 / 1.9M \\  
 \midrule
 prop. ($\*A_U$/$\*A_V$) & 0.57 / 0.63 & 0.77 / 0 & 0.63 / 0.61 & 0.45 / 0 & 0 / 0.8
 & 0 / 0.88 - 0 / 0.71 - 0 / 0.65 \\  
\bottomrule
\end{tabular}
\end{sc}
\end{footnotesize}
\end{table*}

\section{Conclusion} \label{sec:conclusion}
We present a highly efficient method to improve the quality of graph side-information for matrix factorisation. Of the three steps in the algorithm, the initialisation of the latent features and the estimation of the latent features with the updated graph (the E-step) can be performed with any method for matrix completion without SI and with graph SI respectively. With such a small computational cost a graph update (the M-step) to improve quality seems like a valuable step when including graph side-information into matrix factorisation.  Furthermore, we demonstrated the added robustness using our algorithm on real graph side-information. By increasing the number of nearest neighbours for generating graphs from feature side-information our algorithm, GRAEM, improved while GRALS worsened. Our graph update step allows for more noisy graphs to improve the matrix completion accuracy.

Future work on improving the graph update could further improve this method; we showed with simulated data the GLASSO approximation is only moderately successful. 

\section{Acknowledgements}
The research was partly funded by the Academy of Finland grant 313748 and Business Finland grant 211548, computational resources provided by the Aalto Science-IT project.

\newpage


\bibliography{main}
\bibliographystyle{plainnat}

\newpage


\subfile{supplement}

\end{document}

%% file: supplement.tex
\maketitle

\section{Appendix}
\appendix

\section{Posterior of GPMF model}
We derive the posterior of $\*U$, fixing $\*V$, given the data $\{ \*R, \*\Omega \}$ and parameters $\*\Lambda_U$ for the Graph-based prior probabilitic matrix factoriation (GPMF) model. The posterior for $\*V$ follows the same steps with $\*U$ fixed. We start by breaking down the likelihood and prior into scalar operations:
\begin{align}
\log & \; p(\*R \mid \*U, \*V, \alpha, \*\Lambda_U)  \\
 &\propto - \sum_{i=1}^{N} \sum_{j=1}^{M} {\*\Omega_{ij}} \left[\frac{\alpha}{2} \left(\*R_{ij} - \*U_{i:} \*V_{j:}^\top \right)^2\right] - \sum_{d=1}^{D} \frac{1}{2} \*U_{:d}^\top \Lambda_U \*U_{:d} 
 \\
 &= - \frac{\alpha}{2} \sum_{i=1}^{N} \sum_{j=1}^{M} \*\Omega_{ij} \left[ \left(\*R_{ij}^2 -2\*R_{ij}\sum_{d=1}^{D} \*U_{id} \*V_{jd} + \sum_{d=1}^{D}\sum_{d'=1}^{D} \*V_{jd}\*U_{id}\*U_{id'}\*V_{jd'} \right)\right] \\ 
  & \quad \quad - \frac{1}{2} \sum_{d=1}^{D} \sum_{i=1}^{N} \sum_{i'=1}^{N}  \*U_{id} \left[\*\Lambda_U\right]_{ii'} \*U_{i'd} 
  \label{eq:ExpandPosteriorToComponents} \\
&= - \frac{\alpha}{2} \sum_{i=1}^{N} \sum_{j=1}^{M} \*\Omega_{ij}\*R_{ij}^2 - \frac{1}{2} \sum_{i=1}^{N}\left( \alpha \sum_{j=1}^{M} \*\Omega_{ij} \sum_{d=1}^{D}\sum_{d'=1}^{D} \left[ \*V_{jd}\*U_{id}\*U_{id'}\*V_{jd'} \right. \right. \\
 & \left. \left. \quad \quad \quad \quad  -2 \*U_{id} \*V_{jd} \*R_{ij} \right] + \sum_{d=1}^{D} \sum_{i'=1}^{N} \*U_{id} \left[ \*\Lambda_U \right]_{ii'} \*U_{i'd}  \right)  \ . 
 \\ 
\end{align}
Using the scalar expansion we recombine to form the full posterior in scalar form in \Cref{eq:FullPostScalar}, with respect to the vectorization of $\*U$ \Cref{eq:FullMatrixPosteriorWRTvecUColStack} and w.r.t. the vectorization of $\*U^\top$ \Cref{eq:FullMatrixPosteriorWRTvecURowStack}:
\begin{align}
\log & \; p(\*U \mid \*R,  \alpha, \*V, \*\Lambda_U)  \\
& \propto  -\frac{1}{2} \sum_{i=1}^{N} \sum_{j=1}^{M} \sum_{d=1}^{D} \left( \alpha  \*\Omega_{ij}  \left[ \sum_{d'=1}^{D} \*V_{jd}\*U_{id}\*U_{id'}\*V_{jd'}  -2 \*U_{id}\*V_{jd}\*R_{ij} \right]  +  \sum_{i'=1}^{N} \*U_{id} \left[\*\Lambda_U\right]_{ii'} \*U_{i'd} \right) \\
&= -\frac{1}{2}  \sum_{i=1}^{N} \sum_{d=1}^{D} \*U_{id}  \left( \alpha \sum_{j=1}^{M} \*\Omega_{ij} \left[ \sum_{d'=1}^{D} \*V_{jd}U_{id'}\*V_{jd'}  -2 \*V_{jd}\*R_{ij} \right]  +   \sum_{i'=1}^{N} \left[\*\Lambda_U\right]_{ii'}\*U_{i'd}  \right) 
\\
&= -\frac{1}{2}  \sum_{i=1}^{N} \sum_{i'=1}^{N} \sum_{d=1}^{D} \sum_{d'=1}^{D} \left[\*U_{id}  \left( \alpha \sum_{j=1}^{M} [i=i']I_{ij} \*V_{jd}\*V_{jd'}   +   [d=d']\left[\*\Lambda_U\right]_{ii'}  \right)\*U_{i'd'} \right. \label{eq:FullPostScalar} \\
& \left. \quad \quad \quad \quad -2 \alpha \*U_{id}\*V_{jd}\*R_{ij}\right]  
\label{eq:FullMatrixPosteriorWRTvecUScalarSums}\\
&= - \frac{1}{2} \text{vec}(\*U)^\top \left( I_D \otimes \*\Lambda_U + \alpha  \*C  \right) \text{vec}(\*U) - 2 \alpha \text{ Tr}(\*U^\top \*R\*V) \label{eq:FullMatrixPosteriorWRTvecUColStack} \\
&= - \frac{1}{2} \text{vec}(\*U^\top)^\top \left( \*\Lambda_U \otimes I_D  + \alpha \text{ blkdiag} \left( \left\{ \*B_i \right\}_{i = 1}^N \right) \right) \text{vec}(\*U^\top) -2 \alpha \text{ Tr}(\*U \*V^\top \*R^\top) \ , \label{eq:FullMatrixPosteriorWRTvecURowStack}
\end{align}
where $[i=j]$ is Iverson bracket notation where the value is one if the proposition is satisfied and zero otherwise, 
$\otimes$ is the Kronecker product, 
$\text{vec}(\*X)$ stacks the columns of matrix $\*X$ to produce a vector, $\text{Tr}(\*X)$ is the trace of matrix X and finally $I_N$ is an $N \times N$ identity matrix and:
\begin{align}
\*C &= \begin{bmatrix}
\*c(1,1) & \*c(1,2) & \cdots & \*c(1,D)\\
\*c(2,1) & \*c(2,2) & \cdots & \*c(2,D)\\
\vdots & \vdots & \ddots & \vdots\\
\*c(D,1) & \*c(D,2) & \cdots & \*c(D,D)\\
\end{bmatrix} \label{eq:PosteriorCovCMatrix}\\
\*c(d, d') &= \text{diag}\left( \left\{ \sum_{j=1}^{M}\*\Omega_{ij} \*V_{jd} \*V_{jd'} \right\}_{i=1}^N \right) \\
\*B_i &= \sum_{\{j : (i,j) \in \{ \*\Omega = 1 \} \}} \*V_{j:}^\top \*V_{j:} 
\end{align}
Notice that in the posterior when stacking the columns $\text{vec}(\*U)$ in \Cref{eq:FullMatrixPosteriorWRTvecUColStack} the prior precision matrix is a block diagonal matrix and the evidence matrix is a partitioned matrix with each block being diagonal, when stacking the rows  $\text{vec}(\*U^\top)$ the structural pattern is the other way around: the prior is a partitioned matrix of diagonal blocks and the evidence matrix is a block diagonal matrix. It is worth noting that \Cref{eq:FullMatrixPosteriorWRTvecUColStack} and \Cref{eq:FullMatrixPosteriorWRTvecURowStack} both have the structure of a Kronecker sum, $\*A \oplus \*D = \*A \otimes I + I \otimes \*D$.
We look more closely at \Cref{eq:FullMatrixPosteriorWRTvecURowStack}, showing the relation with the scalar summations and the final notation in more detail. Firstly the linear term:
\begin{align}
-2 \alpha \sum_{i=1}^{N} \sum_{d=1}^{D} \*U_{id}\sum_{j=1}^{M}\*\Omega_{ij}\*V_{jd}\*R_{ij} &= -2 \alpha \sum_{i=1}^{N} \sum_{d=1}^{D} \*U_{id} \sum_{j=1}^{M}\*V_{jd}\*R_{ij} 
\\
&= -2 \alpha \sum_{d=1}^{D} \*U_{:d}^\top \*R \*V_{:d} 
\\
&= -2 \alpha \text{ vec}(\*U)^\top \text{vec}(\*R\*V) 
\\
&= -2 \alpha \text{ Tr}(\*U^\top \*R\*V) \\
& = -2 \alpha \sum_{i=1}^{N} \*U_{i:} \*V^\top [\*R_{i:}]^\top  
\\ 
& = -2 \alpha \text{ vec}(\*U^\top)^\top \text{vec}(\*V^\top \*R^\top) \\
& = -2 \alpha \text{ Tr}(\*U \*V^\top \*R^\top) \quad 
\ ,
\end{align}
and the quadratic term:
\begin{align}
\sum_{i=1}^{N} \sum_{i'=1}^{N} \sum_{d=1}^{D} \sum_{d'=1}^{D} \*U_{id} \sum_{j=1}^{M} [i=i'] \left[ \*V_{jd}\*V_{jd'} \right] \*U_{i'd'} &= \text{vec}(\*U)^\top \left[ \sum_{j = 1}^{M} \*V_{j:}^\top \*V_{j:} \otimes I_N \right] \text{vec}(\*U)  
\\
& = \text{vec}(\*U)^\top \left[ \text{ diag}\left\{\sum_{j=1}^M \*V_{j:}^\top \*V_{j:}\right\}_N \right] \text{vec}(\*U)\\
\sum_{i=1}^{N} \sum_{i'=1}^{N} \sum_{d=1}^{D} \sum_{d'=1}^{D} \*U_{id}  [d=d']\left[\*\Lambda_U\right]_{ii'} \*U_{i'd'} &= \text{vec}(\*U)^\top \left[ \*\Lambda_U \otimes I_D \right] \text{vec}(\*U) \\
& = \text{vec}(\*U^\top)^\top \left[ I_D \otimes \*\Lambda_U \right] \text{vec}(\*U^\top) \ .
\end{align}
Having organized the posterior, with respect to $\*U$, into a quadratic and a linear term we can complete the square to find the mean $\*\mu_U^{(n)}$ and precision matrix $\*\Lambda_U^{(n)}$ of the conditional posterior distribution for the matrix $\*U$:
\begin{align}
\*\Lambda_U^{(n)} &= \left[ \*\Lambda_U \otimes I_D \right] + \alpha \text{ blkdiag} \left( \left\{ \*B_i \right\}_{i = 1}^N \right) \\
\*\mu_U^{(n)} &= \left[\*\Lambda_U^{(n)}\right]^{-1} \text{ vec}(\*V^\top R^\top)
\end{align}

or the mean and covariance can be represented as different formulations with scalar sums \eqref{eq:FullMatrixPosteriorWRTvecUScalarSums}, or vectorization of the matrix without transposing \eqref{eq:FullMatrixPosteriorWRTvecUColStack}.

\section{Experiments: further details} \label{appendix:sec:experiments}

We compare our GPMF method to GRALS\footnote{GRALS code: https://github.com/rofuyu/exp-grmf-nips15} \cite{rao2015collaborative}, PMF (GRALS with no graph side-information) \cite{mnih2008probabilistic}, KPMF\footnote{KPMF code: https://people.eecs.berkeley.edu/~tinghuiz/} \cite{zhou2012kernelized} and sRMGCNN\footnote{Recurrent Multi-Graph Neural Networks code: https://github.com/fmonti/mgc} \cite{monti2017geometric}
. KPMF uses the regularised Laplacian graph kernel. We also tested KBMF\footnote{Software: https://github.com/mehmetgonen/kbmf} \cite{gonen2013kernelized}, but with many tuning parameters and a slow learning speed, making parameter tuninig costly and complex, making a similar effort as made for tuning the other algorithms we were not able to attain good results. GRALS, PMF, GPMF (GRAEM, our method), KPMF and KBMF experiments were run on a regular laptop computer: Hewlett Packard EliteBook 840 G3 notebook with Intel Core i5 and 16 GiB memory. sRMGCNN was run on a GPU (NVIDIA Tesla P100, running on a 2x12 core Xeon Dell PowerEdge C4130). For model learning and evaluation we use a the same training and validation set for all models.  We use the same systematic search procedure (similar effort for tuning) for each model for a fair comparison. 

\subsection{Synthetic Data}
Default settings for the experiments, if no other details are mentioned, are a $400 \times 400$ data matrix with 7 percent observed values, a graph fidelity of 0.7, observation noise $\sigma^2= 0.01$, 40 latent feature dimensions and noise between similar latent features 0.0001. We use the graphs to create the latent feature matrices $\*U$ and $\*V$  according to the GPMF model. Sampling of observed entries for training and validation is according to a non-uniform distribution; to avoid rows and columns having similar numbers of observations, we use a multinomial distribution with Dirichlet prior). Each experiment setting is run five times for each model and an average is reported with the standard deviation as the height of the error bar.

\subsection{Real data experiments}

\paragraph{Flixster (3k).} A three thousand dimensional subset matrix from the original Flixster dataset\footnote{\label{ftn:geoDLRef}https://github.com/fmonti/mgcnn} as in \cite{monti2017geometric}. Where graph side information is constructed from the scores of the original matrix: $\*G_U$ with 59354 edges and $\*G_V$ 50918 edges.

\paragraph{Douban (3k).} A three thousand dimensional subset matrix from the orignal Douban dataset\footnote{See footnote 1} as in \cite{monti2017geometric}. Where user graph side-information is a social network with 2688 edges.

\paragraph{MovieLens 100k and 20M.} The GroupLens official MovieLens\footnote{https://grouplens.org/datasets/movielens/} 100k and 20M datasets \cite{harper2015movielens}. For MovieLens 100k graph side-information is constructed for users, based on user demographic information using k-nearest neighbour (kNN) algorithm with ten neighbours. For MoveLens 20M graph side-information is constructed using kNN with k=\{10,20,40\} based on movie genre data with 492956, 962644 and 1870508 edges respectively.

\paragraph{Epinions.} We take the Epinions\footnote{www.epinions.com} dataset as described in KPMF, but we use a much larger data size\footnote{https://www.cse.msu.edu/\textasciitilde tangjili/trust.html , https://www.cse.msu.edu/\textasciitilde tangjili/datasetcode/epinions.zip} (22164 x 296277) with user trust network data (22164 x 22164) \cite{tang2012mtrust}. The dataset is extremely sparse (9.8312e-05 proportion of observed entries), and distributed un-uniformly, making this a difficult problem. 

\paragraph{YahooMusic} The official Yahoo Music ratings data from the KDD cup \cite{dror2011yahoo} as used in \cite{rao2015collaborative} to demonstrate scalability. We construct the graph with exact kNN on the music covariate data (artist,genre,album) with ten neighbours. This results in a very sparse graph, likely connecting many music tracks from the same artist and in the same album only.

\paragraph{Model tuning} Data is split into test and validation. We use a procedure of parameter searching that we repeat for each model. PMF observation noise is fixed at $\sigma^2=1$. PMF, GRALS and GPMF use the samge CG iterations configuration (CG). GPMF uses $\tau=0$ for thresholding. KPMF uses the regularised Laplacian graph kernel with graph strength $\gamma$ and learning rate $\epsilon$. KBMF is trained with uninformative priors: $(\alpha_{\lambda}=1,\beta_{\lambda}=1)$, changing these values we saw no improvents; with at most one graph for each kernel multi-kernel parameters do not require tuning.

\begin{table}[ht!]
\caption{Model parameter tuning for real world experiments}
\label{tab:TuningRealDataExperiments}
\begin{small}
\begin{sc}
\hskip-1.3cm\begin{tabular}{lrrrr}
\toprule
&  Flixster &Douban&MovieLens& Epinions  \\
  &  (3k) &  (3k) & 100k &  \\
\midrule
D=  &  10 & 10  & 10 & 10  \\
\midrule
PMF (\scriptsize{$\sigma_U,\sigma_V,CG$}) & 0.1,0.1,1 & 5,5,1 & 1.2,1.2,2 & 0.75,0.2,3 \\
GRALS (\scriptsize{$\lambda_L,\lambda_U,\lambda_V$}) & 3,0.5,0.5 & 8,2,5 & 0.1,0.01,0.01 & 0.01,0.01,0.02 \\
GPMF (\scriptsize{ $\sigma_2,\lambda_L,\lambda_U,\lambda_V$}) & 1,5,1,1 & 0.5,5,2,5  & 0.05,1,0.05,0.05 & 0.1,0.1,0.1,0.1 \\
KPMF ($\sigma^2,\epsilon,\gamma_U,\gamma_V$) & $0.1,10^{-6},1,1$  & $0.07,10^{-6},100,100$ & $0.2,10^{-5},1.1,1.1$  & N/A  \\
KBMF ($\gamma_U,\gamma_V,\sigma_g,\sigma_y$) &  1,1,0.1,1 & 1,1,0.2,1 & 0.35,0.3,0.1,0.15 & N/A \\
\midrule
& Yahoo \vline &  MovieLens 20M & & \\
 &  Music \vline &  10NN& 20NN & 40NN\\
\midrule
D=  & 20  & 10 & 10 & 10 \\
\midrule
PMF (\scriptsize{$\sigma_U,\sigma_V,CG$}) & 10,10 & 1.25,12.5,10 & 1.25,12.5,10 & 1.25,12.5,10 \\
GRALS (\scriptsize{$\lambda_L,\lambda_U,\lambda_V$}) & 100,200  & 5,0.5,0.01  & 1,1,10 & 1,1,10 \\
GPMF (\scriptsize{ $\sigma_2,\lambda_L,\lambda_U,\lambda_V$}) & 10,10,100,200 & 0.05,5,0.1,0.01 & 0.5,5,0.1,0.1 & 0.5,2.5,0.01,0.01 \\
\bottomrule
\end{tabular}
\end{sc}
\end{small}
\end{table}

\paragraph{General use of updated graph} We believe that removing the \textit{contested} edges improves the graph for this task in general for any model, not just for GPMF.  To this end we tested this by using the full graph vs. updated graph for KPMF \cite{zhou2012kernelized} . We plot an example of improved covergence speed and accuracy for KPMF in \Cref{fig:ComputationalComplexityCompareMovienLens100k}. 
\begin{figure}[ht!]
\centering
\includegraphics[width=0.75\columnwidth]{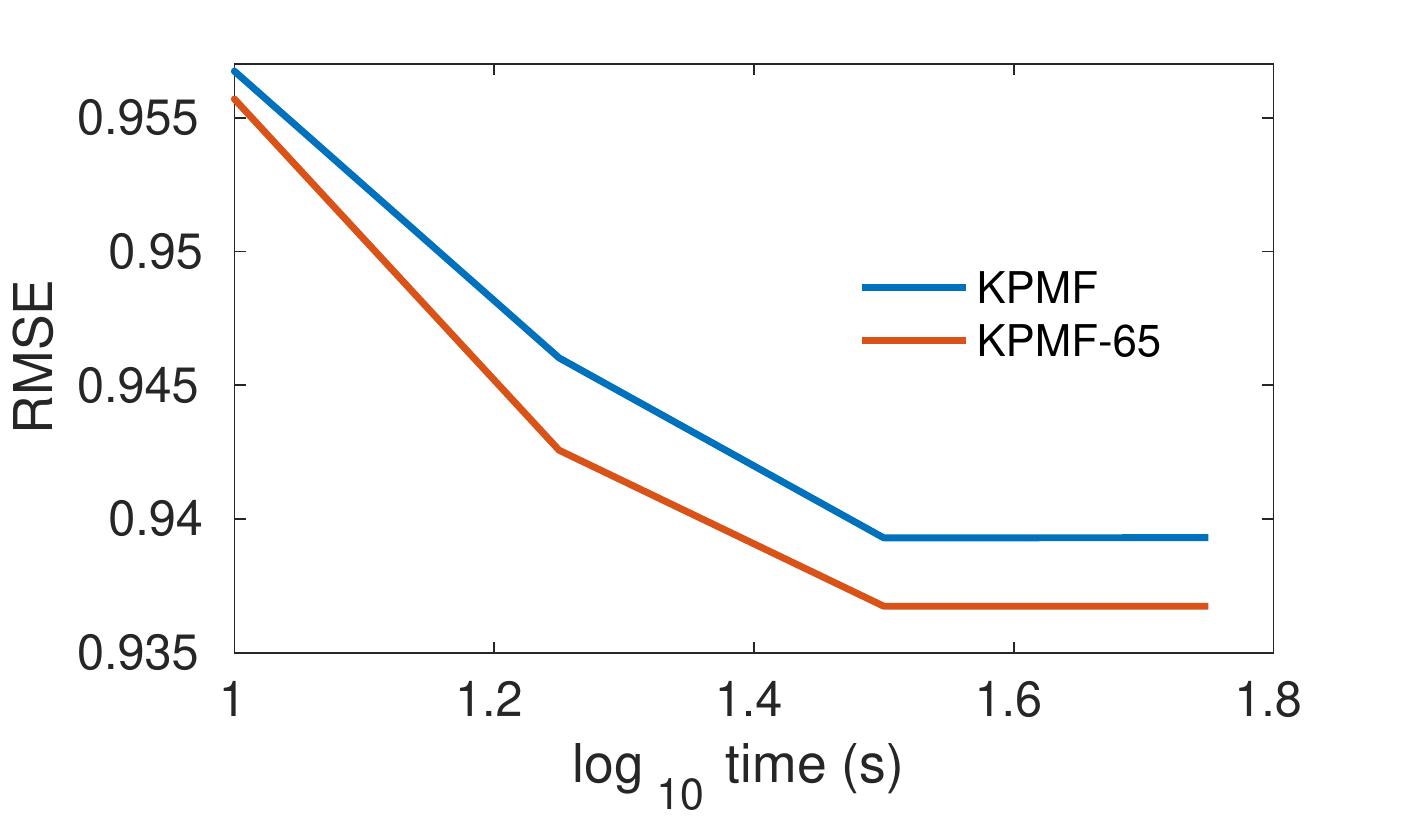}
\caption{Convergence time on MovieLens 100k. We provide an updated graph with 65\% of the edges learnt with GPMF to 
another graph-regularised matrix factorisation method (KPMF \cite{zhou2012kernelized}, to show that the optimised graph improves the convergence speed and precision of arbitraty algorithms for the graph regularised matrix completion problem.}
\label{fig:ComputationalComplexityCompareMovienLens100k}
\end{figure}




%% file: main.bbl
\begin{thebibliography}{57}
\providecommand{\natexlab}[1]{#1}
\providecommand{\url}[1]{\texttt{#1}}
\expandafter\ifx\csname urlstyle\endcsname\relax
  \providecommand{\doi}[1]{doi: #1}\else
  \providecommand{\doi}{doi: \begingroup \urlstyle{rm}\Url}\fi

\bibitem[Adar and Re(2007)]{adar2007managing}
Eytan Adar and Christopher Re.
\newblock Managing uncertainty in social networks.
\newblock \emph{IEEE Data Eng. Bull.}, 30\penalty0 (2):\penalty0 15--22, 2007.

\bibitem[Ahn et~al.(2018)Ahn, Lee, Cha, and Suh]{ahn2018binary}
Kwangjun Ahn, Kangwook Lee, Hyunseung Cha, and Changho Suh.
\newblock Binary rating estimation with graph side information.
\newblock In \emph{Advances in Neural Information Processing Systems}, pages
  4272--4283, 2018.

\bibitem[Asthana et~al.(2004)Asthana, King, Gibbons, and
  Roth]{asthana2004predicting}
Saurabh Asthana, Oliver~D King, Francis~D Gibbons, and Frederick~P Roth.
\newblock Predicting protein complex membership using probabilistic network
  reliability.
\newblock \emph{Genome research}, 14\penalty0 (6):\penalty0 1170--1175, 2004.

\bibitem[Berg et~al.(2017)Berg, Kipf, and Welling]{berg2017graph}
Rianne van~den Berg, Thomas~N Kipf, and Max Welling.
\newblock Graph convolutional matrix completion.
\newblock \emph{arXiv preprint arXiv:1706.02263}, 2017.

\bibitem[Bishop(2006)]{bishop2006PRML}
Christopher~M. Bishop.
\newblock \emph{Pattern Recognition and Machine Learning}.
\newblock Springer, 2006.

\bibitem[Brunet et~al.(2004)Brunet, Tamayo, Golub, and
  Mesirov]{brunet2004metagenes}
Jean-Philippe Brunet, Pablo Tamayo, Todd~R Golub, and Jill~P Mesirov.
\newblock Metagenes and molecular pattern discovery using matrix factorization.
\newblock \emph{Proceedings of the national academy of sciences}, 101\penalty0
  (12):\penalty0 4164--4169, 2004.

\bibitem[Cai et~al.(2011)Cai, He, Han, and Huang]{cai2011graph}
Deng Cai, Xiaofei He, Jiawei Han, and Thomas~S Huang.
\newblock Graph regularized nonnegative matrix factorization for data
  representation.
\newblock \emph{IEEE Transactions on Pattern Analysis and Machine
  Intelligence}, 33\penalty0 (8):\penalty0 1548--1560, 2011.

\bibitem[Candes and Plan(2010)]{candes2010matrix}
Emmanuel~J Candes and Yaniv Plan.
\newblock Matrix completion with noise.
\newblock \emph{Proceedings of the IEEE}, 98\penalty0 (6):\penalty0 925--936,
  2010.

\bibitem[Cand{\`e}s and Recht(2009)]{candes2009exact}
Emmanuel~J Cand{\`e}s and Benjamin Recht.
\newblock Exact matrix completion via convex optimization.
\newblock \emph{Foundations of Computational mathematics}, 9\penalty0
  (6):\penalty0 717, 2009.

\bibitem[Chiang et~al.(2015)Chiang, Hsieh, and Dhillon]{chiang2015matrix}
Kai-Yang Chiang, Cho-Jui Hsieh, and Inderjit~S Dhillon.
\newblock Matrix completion with noisy side information.
\newblock In C.~Cortes, N.~D. Lawrence, D.~D. Lee, M.~Sugiyama, and R.~Garnett,
  editors, \emph{Advances in Neural Information Processing Systems 28}, pages
  3447--3455. Curran Associates, Inc., 2015.
\newblock URL
  \url{http://papers.nips.cc/paper/5940-matrix-completion-with-noisy-side-information.pdf}.

\bibitem[Chiang et~al.(2018)Chiang, Dhillon, and Hsieh]{Chiang2018NoisySI}
Kai-Yang Chiang, Inderjit~S. Dhillon, and Cho-Jui Hsieh.
\newblock Using side information to reliably learn low-rank matrices from
  missing and corrupted observations.
\newblock \emph{Journal of Machine Learning Research}, 19\penalty0
  (76):\penalty0 1--35, 2018.
\newblock URL \url{http://jmlr.org/papers/v19/17-112.html}.

\bibitem[Davenport and Romberg(2016)]{davenport2016overview}
Mark~A Davenport and Justin Romberg.
\newblock An overview of low-rank matrix recovery from incomplete observations.
\newblock \emph{IEEE Journal of Selected Topics in Signal Processing},
  10\penalty0 (4):\penalty0 608--622, 2016.

\bibitem[Davis et~al.(2004)Davis, Gilbert, Larimore, and Ng]{davis2004column}
Timothy~A Davis, John~R Gilbert, Stefan~I Larimore, and Esmond~G Ng.
\newblock A column approximate minimum degree ordering algorithm.
\newblock \emph{ACM Transactions on Mathematical Software (TOMS)}, 30\penalty0
  (3):\penalty0 353--376, 2004.

\bibitem[Dong et~al.(2016)Dong, Thanou, Frossard, and
  Vandergheynst]{dong2016learning}
Xiaowen Dong, Dorina Thanou, Pascal Frossard, and Pierre Vandergheynst.
\newblock Learning laplacian matrix in smooth graph signal representations.
\newblock \emph{IEEE Transactions on Signal Processing}, 64\penalty0
  (23):\penalty0 6160--6173, 2016.

\bibitem[Dror et~al.(2011)Dror, Koenigstein, Koren, and Weimer]{dror2011yahoo}
Gideon Dror, Noam Koenigstein, Yehuda Koren, and Markus Weimer.
\newblock The yahoo! music dataset and kdd-cup'11.
\newblock In \emph{Proceedings of the 2011 International Conference on KDD Cup
  2011-Volume 18}, pages 3--18. JMLR. org, 2011.

\bibitem[Egilmez et~al.(2016)Egilmez, Pavez, and Ortega]{egilmez2016graph}
Hilmi~E Egilmez, Eduardo Pavez, and Antonio Ortega.
\newblock Graph learning with laplacian constraints: Modeling attractive
  gaussian markov random fields.
\newblock In \emph{Signals, Systems and Computers, 2016 50th Asilomar
  Conference on}, pages 1470--1474. IEEE, 2016.

\bibitem[Egilmez et~al.(2017)Egilmez, Pavez, and Ortega]{egilmez2017graph}
Hilmi~E Egilmez, Eduardo Pavez, and Antonio Ortega.
\newblock Graph learning from data under laplacian and structural constraints.
\newblock \emph{IEEE Journal of Selected Topics in Signal Processing},
  11\penalty0 (6):\penalty0 825--841, 2017.

\bibitem[Fattahi and Sojoudi(2017)]{fattahi2017graphical}
Salar Fattahi and Somayeh Sojoudi.
\newblock Graphical lasso and thresholding: Equivalence and closed-form
  solutions.
\newblock \emph{arXiv preprint arXiv:1708.09479}, 2017.

\bibitem[Fattahi and Sojoudi(2019)]{fattahi2019graphical}
Salar Fattahi and Somayeh Sojoudi.
\newblock Graphical lasso and thresholding: equivalence and closed-form
  solutions.
\newblock \emph{The Journal of Machine Learning Research}, 20\penalty0
  (1):\penalty0 364--407, 2019.

\bibitem[G{\"o}nen et~al.(2013)G{\"o}nen, Khan, and Kaski]{gonen2013kernelized}
Mehmet G{\"o}nen, Suleiman Khan, and Samuel Kaski.
\newblock Kernelized bayesian matrix factorization.
\newblock In \emph{International Conference on Machine Learning}, pages
  864--872, 2013.

\bibitem[Grechkin et~al.(2015)Grechkin, Fazel, Witten, and
  Lee]{grechkin2015pathway}
Maxim Grechkin, Maryam Fazel, Daniela Witten, and Su-In Lee.
\newblock Pathway graphical lasso.
\newblock In \emph{Twenty-Ninth AAAI Conference on Artificial Intelligence},
  2015.

\bibitem[Haghani and Keyvanpour(2017)]{haghani2017systemic}
Sogol Haghani and Mohammad~Reza Keyvanpour.
\newblock A systemic analysis of link prediction in social network.
\newblock \emph{Artificial Intelligence Review}, pages 1--35, 2017.

\bibitem[Harper and Konstan(2015)]{harper2015movielens}
F~Maxwell Harper and Joseph~A Konstan.
\newblock The movielens datasets: History and context.
\newblock \emph{Transactions on Interactive Intelligent Systems}, 5\penalty0
  (4), 2015.

\bibitem[Hartford et~al.(2018)Hartford, Graham, Leyton-Brown, and
  Ravanbakhsh]{hartford2018deep}
Jason Hartford, Devon~R Graham, Kevin Leyton-Brown, and Siamak Ravanbakhsh.
\newblock Deep models of interactions across sets.
\newblock \emph{arXiv preprint arXiv:1803.02879}, 2018.

\bibitem[Hastie et~al.(2009)Hastie, Tibshirani, and
  Friedman]{hastie2009elements}
T.~Hastie, R.~Tibshirani, and J.~Friedman.
\newblock \emph{The Elements of Statistical Learning: Data Mining, Inference,
  and Prediction, Second Edition}.
\newblock Springer Series in Statistics. Springer New York, 2009.

\bibitem[Jacoby and Brown(2018)]{jacoby2018future}
Edgar Jacoby and JB~Brown.
\newblock The future of computational chemogenomics.
\newblock In \emph{Computational Chemogenomics}, pages 425--450. Springer,
  2018.

\bibitem[Kalaitzis et~al.(2013)Kalaitzis, Lafferty, Lawrence, and
  Zhou]{kalaitzis2013bigraphical}
Alfredo Kalaitzis, John Lafferty, Neil Lawrence, and Shuheng Zhou.
\newblock The bigraphical lasso.
\newblock In \emph{International Conference on Machine Learning}, pages
  1229--1237, 2013.

\bibitem[Koren et~al.(2009)Koren, Bell, and Volinsky]{koren2009matrix}
Yehuda Koren, Robert Bell, and Chris Volinsky.
\newblock Matrix factorization techniques for recommender systems.
\newblock \emph{Computer}, 2009.

\bibitem[Lauritzen(1996)]{lauritzen1996graphical}
S.L. Lauritzen.
\newblock \emph{Graphical Models}.
\newblock Oxford science publications. Clarendon Press, 1996.

\bibitem[Li et~al.(2014)Li, Gao, Guo, Du, Li, and Zhang]{li2014lrbm}
Kang Li, Jing Gao, Suxin Guo, Nan Du, Xiaoyi Li, and Aidong Zhang.
\newblock Lrbm: A restricted boltzmann machine based approach for
  representation learning on linked data.
\newblock In \emph{2014 IEEE International Conference on Data Mining}, pages
  300--309. IEEE, 2014.

\bibitem[Liu et~al.(2014)Liu, Chakraborty, Li, Liu, Lozano,
  et~al.]{liu2014bayesian}
Fei Liu, Sounak Chakraborty, Fan Li, Yan Liu, Aurelie~C Lozano, et~al.
\newblock Bayesian regularization via graph laplacian.
\newblock \emph{Bayesian Analysis}, 9\penalty0 (2):\penalty0 449--474, 2014.

\bibitem[Liu et~al.(2018)Liu, Safavi, Dighe, and Koutra]{liu2018graph}
Yike Liu, Tara Safavi, Abhilash Dighe, and Danai Koutra.
\newblock Graph summarization methods and applications: A survey.
\newblock \emph{ACM Computing Surveys (CSUR)}, 51\penalty0 (3):\penalty0 62,
  2018.

\bibitem[Ma et~al.(2011)Ma, Zhou, Liu, Lyu, and King]{ma2011recommender}
Hao Ma, Dengyong Zhou, Chao Liu, Michael~R Lyu, and Irwin King.
\newblock Recommender systems with social regularization.
\newblock In \emph{Proceedings of the fourth ACM international conference on
  Web search and data mining}, pages 287--296. ACM, 2011.

\bibitem[Mazumder and Hastie(2012)]{mazumder2012graphical}
Rahul Mazumder and Trevor Hastie.
\newblock The graphical lasso: New insights and alternatives.
\newblock \emph{Electronic journal of statistics}, 6:\penalty0 2125, 2012.

\bibitem[McPherson et~al.(2001)McPherson, Smith-Lovin, and
  Cook]{mcpherson2001birds}
Miller McPherson, Lynn Smith-Lovin, and James~M Cook.
\newblock Birds of a feather: Homophily in social networks.
\newblock \emph{Annual review of sociology}, 27\penalty0 (1):\penalty0
  415--444, 2001.

\bibitem[Mehta and Rana(2017)]{mehta2017review}
Rachana Mehta and Keyur Rana.
\newblock A review on matrix factorization techniques in recommender systems.
\newblock In \emph{2017 2nd International Conference on Communication Systems,
  Computing and IT Applications (CSCITA)}, pages 269--274. IEEE, 2017.

\bibitem[Mnih and Salakhutdinov(2008)]{mnih2008probabilistic}
Andriy Mnih and Ruslan~R Salakhutdinov.
\newblock Probabilistic matrix factorization.
\newblock In \emph{Advances in neural information processing systems}, pages
  1257--1264, 2008.

\bibitem[Monti et~al.(2017)Monti, Bronstein, and Bresson]{monti2017geometric}
Federico Monti, Michael Bronstein, and Xavier Bresson.
\newblock Geometric matrix completion with recurrent multi-graph neural
  networks.
\newblock In \emph{Advances in Neural Information Processing Systems}, pages
  3697--3707, 2017.

\bibitem[Nguyen and Mamitsuka(2012)]{nguyen2012latent}
Canh~Hao Nguyen and Hiroshi Mamitsuka.
\newblock Latent feature kernels for link prediction on sparse graphs.
\newblock \emph{IEEE transactions on neural networks and learning systems},
  23\penalty0 (11):\penalty0 1793--1804, 2012.

\bibitem[Rao et~al.(2015)Rao, Yu, Ravikumar, and Dhillon]{rao2015collaborative}
Nikhil Rao, Hsiang-Fu Yu, Pradeep~K Ravikumar, and Inderjit~S Dhillon.
\newblock Collaborative filtering with graph information: Consistency and
  scalable methods.
\newblock In \emph{Advances in neural information processing systems}, pages
  2107--2115, 2015.

\bibitem[Rue and Held(2005)]{rue2005gaussian}
H.~Rue and L.~Held.
\newblock \emph{Gaussian Markov Random Fields: Theory and Applications}.
\newblock CRC Press, 2005.

\bibitem[Sardianos et~al.(2019)Sardianos, Papadatos, and
  Varlamis]{sardianos2019optimizing}
Christos Sardianos, Grigorios~Ballas Papadatos, and Iraklis Varlamis.
\newblock Optimizing parallel collaborative filtering approaches for improving
  recommendation systems performance.
\newblock \emph{Information}, 10\penalty0 (5):\penalty0 155, 2019.

\bibitem[Schacke(2004)]{schacke2004kronecker}
Kathrin Schacke.
\newblock On the kronecker product.
\newblock \emph{Master's thesis, University of Waterloo}, 2004.

\bibitem[Singla and Richardson(2008)]{singla2008yes}
Parag Singla and Matthew Richardson.
\newblock Yes, there is a correlation:-from social networks to personal
  behavior on the web.
\newblock In \emph{Proceedings of the 17th international conference on World
  Wide Web}, pages 655--664. ACM, 2008.

\bibitem[Sojoudi(2016{\natexlab{a}})]{sojoudi2016equivalence}
Somayeh Sojoudi.
\newblock Equivalence of graphical lasso and thresholding for sparse graphs.
\newblock \emph{The Journal of Machine Learning Research}, 17\penalty0
  (1):\penalty0 3943--3963, 2016{\natexlab{a}}.

\bibitem[Sojoudi(2016{\natexlab{b}})]{sojoudi2016graphical}
Somayeh Sojoudi.
\newblock Graphical lasso and thresholding: Conditions for equivalence.
\newblock In \emph{Decision and Control (CDC)}, pages 7042--7048. IEEE,
  2016{\natexlab{b}}.

\bibitem[Stein-O’Brien et~al.(2018)Stein-O’Brien, Arora, Culhane, Favorov,
  Garmire, Greene, Goff, Li, Ngom, Ochs, et~al.]{stein2018enter}
Genevieve~L Stein-O’Brien, Raman Arora, Aedin~C Culhane, Alexander~V Favorov,
  Lana~X Garmire, Casey~S Greene, Loyal~A Goff, Yifeng Li, Aloune Ngom,
  Michael~F Ochs, et~al.
\newblock Enter the matrix: factorization uncovers knowledge from omics.
\newblock \emph{Trends in Genetics}, 2018.

\bibitem[Tang et~al.(2012)Tang, Gao, and Liu]{tang2012mtrust}
Jiliang Tang, Huiji Gao, and Huan Liu.
\newblock mtrust: discerning multi-faceted trust in a connected world.
\newblock In \emph{Proceedings of the fifth ACM international conference on Web
  search and data mining}, pages 93--102. ACM, 2012.

\bibitem[Xu and Yin(2013)]{xu2013block}
Yangyang Xu and Wotao Yin.
\newblock A block coordinate descent method for regularized multiconvex
  optimization with applications to nonnegative tensor factorization and
  completion.
\newblock \emph{SIAM Journal on imaging sciences}, 6\penalty0 (3):\penalty0
  1758--1789, 2013.

\bibitem[Xue et~al.(2017)Xue, Zhang, and Cai]{xue2017depth}
Hongyang Xue, Shengming Zhang, and Deng Cai.
\newblock Depth image inpainting: Improving low rank matrix completion with low
  gradient regularization.
\newblock \emph{IEEE Transactions on Image Processing}, 26\penalty0
  (9):\penalty0 4311--4320, 2017.

\bibitem[Yao and Li(2018)]{yao2018convolutional}
Kai-Lang Yao and Wu-Jun Li.
\newblock Convolutional geometric matrix completion.
\newblock \emph{arXiv preprint arXiv:1803.00754}, 2018.

\bibitem[Zakeri et~al.(2018)Zakeri, Simm, Arany, ElShal, and
  Moreau]{zakeri2018gene}
Pooya Zakeri, Jaak Simm, Adam Arany, Sarah ElShal, and Yves Moreau.
\newblock Gene prioritization using bayesian matrix factorization with genomic
  and phenotypic side information.
\newblock \emph{Bioinformatics}, 34\penalty0 (13):\penalty0 i447--i456, 2018.

\bibitem[Zhang et~al.(2018)Zhang, Fattahi, and
  Sojoudi]{zhang2018largescaleprec}
Richard Zhang, Salar Fattahi, and Somayeh Sojoudi.
\newblock Large-scale sparse inverse covariance estimation via thresholding and
  max-det matrix completion.
\newblock In \emph{Proceedings of the 35th International Conference on Machine
  Learning}, pages 5766--5775. PMLR, 2018.

\bibitem[Zhao et~al.(2017)Zhao, Du, and Buntine]{zhao2017leveraging}
He~Zhao, Lan Du, and Wray Buntine.
\newblock Leveraging node attributes for incomplete relational data.
\newblock In \emph{Proceedings of the 34th International Conference on Machine
  Learning-Volume 70}, pages 4072--4081. JMLR. org, 2017.

\bibitem[{Zhao} et~al.(2015){Zhao}, {Zhang}, {He}, and {Ng}]{zhao2015expert}
Z.~{Zhao}, L.~{Zhang}, X.~{He}, and W.~{Ng}.
\newblock Expert finding for question answering via graph regularized matrix
  completion.
\newblock \emph{IEEE Transactions on Knowledge and Data Engineering},
  27\penalty0 (4):\penalty0 993--1004, April 2015.
\newblock ISSN 1041-4347.
\newblock \doi{10.1109/TKDE.2014.2356461}.

\bibitem[Zheng et~al.(2013)Zheng, Ding, Mamitsuka, and
  Zhu]{zheng2013collaborative}
Xiaodong Zheng, Hao Ding, Hiroshi Mamitsuka, and Shanfeng Zhu.
\newblock Collaborative matrix factorization with multiple similarities for
  predicting drug-target interactions.
\newblock In \emph{Proceedings of the 19th ACM SIGKDD international conference
  on Knowledge discovery and data mining}, pages 1025--1033. ACM, 2013.

\bibitem[Zhou et~al.(2012)Zhou, Shan, Banerjee, and Sapiro]{zhou2012kernelized}
Tinghui Zhou, Hanhuai Shan, Arindam Banerjee, and Guillermo Sapiro.
\newblock Kernelized probabilistic matrix factorization: Exploiting graphs and
  side information.
\newblock In \emph{Proceedings of the 2012 SIAM International Conference on
  Data Mining}, pages 403--414. SIAM, 2012.

\end{thebibliography}
